%% file: main.tex
\documentclass[10pt,twocolumn,twoside]{IEEEtran}

\ifCLASSINFOpdf
\else
\fi

\usepackage{amsmath}
\usepackage{amssymb}
\usepackage{amsthm}

\makeatletter
\let\MYcaption\@makecaption
\makeatother

\usepackage[font=footnotesize]{subcaption}

\makeatletter
\let\@makecaption\MYcaption
\makeatother
\usepackage{booktabs}
\usepackage{multirow}
\newtheorem{theorem}{Theorem}[section]
\newtheorem{proposition}[theorem]{Proposition}

\newtheorem{lemma}[theorem]{Lemma}

\theoremstyle{definition}
\newtheorem{definition}[theorem]{Definition}
\newtheorem{problem}[theorem]{Problem}

\theoremstyle{remark}

\theoremstyle{definition}

\usepackage{tikz}
\usetikzlibrary{shapes.geometric}
\usetikzlibrary{arrows.meta,arrows}
\usetikzlibrary{automata,positioning}
\usepackage{tkz-berge}

\usepackage{epstopdf}
\usepackage{color}
\usepackage{comment}
\usepackage[flushleft]{threeparttable}
\usepackage{array}
\usepackage{graphicx}
\usepackage{booktabs}
\usepackage{pgfplots}
\usepackage{url}
\usepackage{tikz}
\usepackage{multirow}
\usetikzlibrary{patterns}
\usetikzlibrary{arrows.meta}
\input{tikzset.tex}
\usepackage[left=2cm,right=2cm,top=2cm,bottom=2cm]{geometry}
\usepackage[ruled,vlined,linesnumbered]{algorithm2e}
\usepackage{cite}
\usepackage{xcolor}
\usepackage{ifthen,xkeyval,calc}

\usepackage{optidef}

\DeclareMathOperator*{\argmin}{arg\,min}

\renewcommand{\thesection}{\Roman{section}}
\renewcommand{\thesubsection}{\thesection.\Alph{subsection}}
\usepackage{algpseudocode}

\urldef{\smith}\url{stephen.smith@uwaterloo.ca}
\urldef{\florence}\url{f4tsang@uwaterloo.ca}
\urldef{\tristan}\url{tlwalker@uwaterloo.ca}
\urldef{\armin}\url{a6sadegh@uwaterloo.ca}
\urldef{\ryan}\url{ryan.macdonald@uwaterloo.ca}

\title{\LARGE \bf
LAMP: Learning a Motion Policy \\ to Repeatedly Navigate in an Uncertain Environment}

\author{Florence Tsang, Tristan Walker, Ryan A.\ MacDonald, Armin Sadeghi, and Stephen L.\ Smith
\thanks{This research is partially supported by the Natural Sciences and Engineering Research Council of Canada (NSERC). }
\thanks{T.\ Walker, A.\ Sadeghi and S.\ L.\ Smith are with the Department of Electrical and Computer Engineering, University of Waterloo, Waterloo ON, N2L 3G1    Canada (\tristan; \armin; \smith)} 
\thanks{F.\ Tsang and R.\ A.\ MacDonald were with the Department of Electrical and Computer Engineering at the University of Waterloo when this research was conducted 
 (\florence; \ryan)}}

\pgfplotsset{compat=1.15}
\begin{document}
\maketitle
\begin{abstract}
Mobile robots are often tasked with repeatedly navigating through an environment whose traversability changes over time.  These changes may exhibit some hidden structure, which can be learned. Many studies consider reactive algorithms for online planning, however, these algorithms do not take advantage of the past executions of the navigation task for future tasks. In this paper, we formalize the problem of minimizing the total expected cost to perform multiple start-to-goal navigation tasks on a roadmap by introducing  the  Learned  Reactive  Planning  Problem. We propose a method that captures information from past executions to learn a motion policy to  handle  obstacles that  the  robot  has  seen  before. We propose the LAMP framework, which integrates the generated motion policy with an  existing navigation stack. Finally, an extensive set of experiments in simulated and real-world environments show that the proposed method outperforms the state-of-the-art algorithms by $10\%$ to $40\%$ in terms of expected time to travel from start to goal. We also evaluate the robustness of the proposed method in the presence of localization and mapping errors on a real robot.
\end{abstract}

\emph{keywords}: Motion and Path Planning, Learning and Adaptive Systems, Reactive and Sensor-Based Planning, and Active Sensing

\section{Introduction} 
\label{sec:intro}

A common approach to dealing with uncertainties while navigating is to plan a shortest path given the current map. If the path is blocked, then the map is updated and path is re-planned. There are many shortest path algorithms that can be utilized to accomplish this (see  \cite{lavalle2000rapidly},\cite{kuffner2000rrt},\cite{devaurs2011parallelizing}, \cite{kavraki1996probabilistic}). This is known as an \emph{optimistic} policy, which is suitable for situations where the robot is not expected to operate in the same environment for a long period of time (e.g. search and rescue missions, exploration missions). However, from a long-term operation perspective (e.g. deployment in a mall or warehouse), this policy is not ideal:  It only considers the obstacles given by the map and other obstacles within its immediate sensor range when forming its strategy, without considering past executions, which could result in consistently `bad' routes to the goal.

Figure \ref{fig:optbad} illustrates an example of an optimistic policy performing poorly. Consider the case where the robot operates in the environment shown in Figure~\ref{subfig:optbad1} 75\% of the time and operates in the environment shown in Figure \ref{subfig:optbad2} otherwise. An optimistic approach will always attempt the orange path in environment (b), resulting in a considerably longer travel distance than the optimal path, shown in blue. One could plan to always execute the blue path, but that would be considerably longer than optimal when the robot is in environment (c). However, what if the robot eventually learns to take the blue path only in environment (b) as it realizes that the door on the right will never be open if the door on the left is closed?

In this paper, we present the Learning a Motion Policy (LAMP) framework which uses an alternative strategy; one that can learn and exploit such hidden properties of a real-world environment, so that over time, the expected cost of the path to reach the goal is minimized. The LAMP framework is designed to be easily deployable on a real robot such as the one shown in Figure \ref{fig:front_image}.

\begin{figure}
        \centering
    \begin{subfigure}{\linewidth}
        \includegraphics[width=\linewidth]{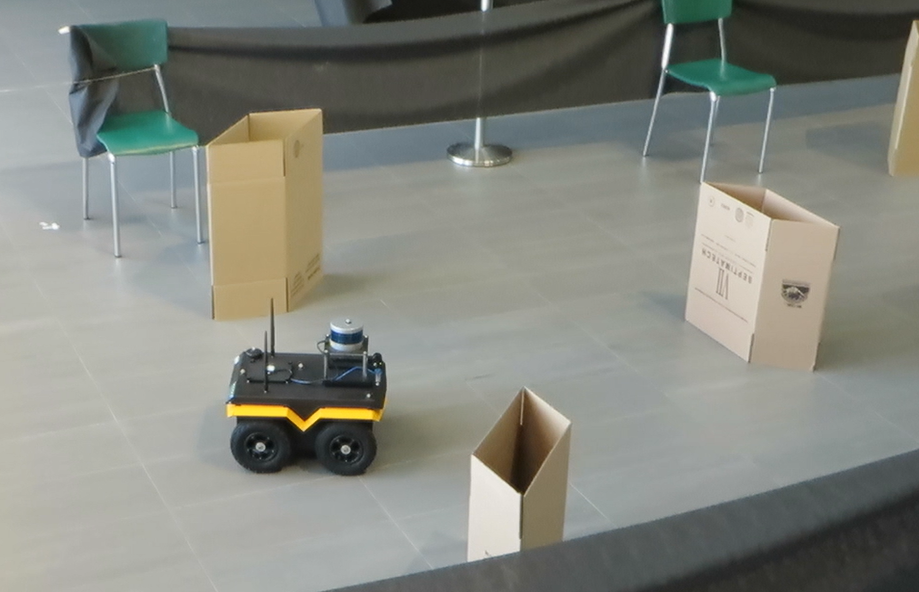}
        \caption{The Jackal using LAMP to navigate a changing maze.}
        \label{fig:front_image}
    \end{subfigure}
    \centering
    \begin{subfigure}{.49\linewidth}
        \begin{center}
        \includegraphics[angle=-90,width=.95\linewidth]{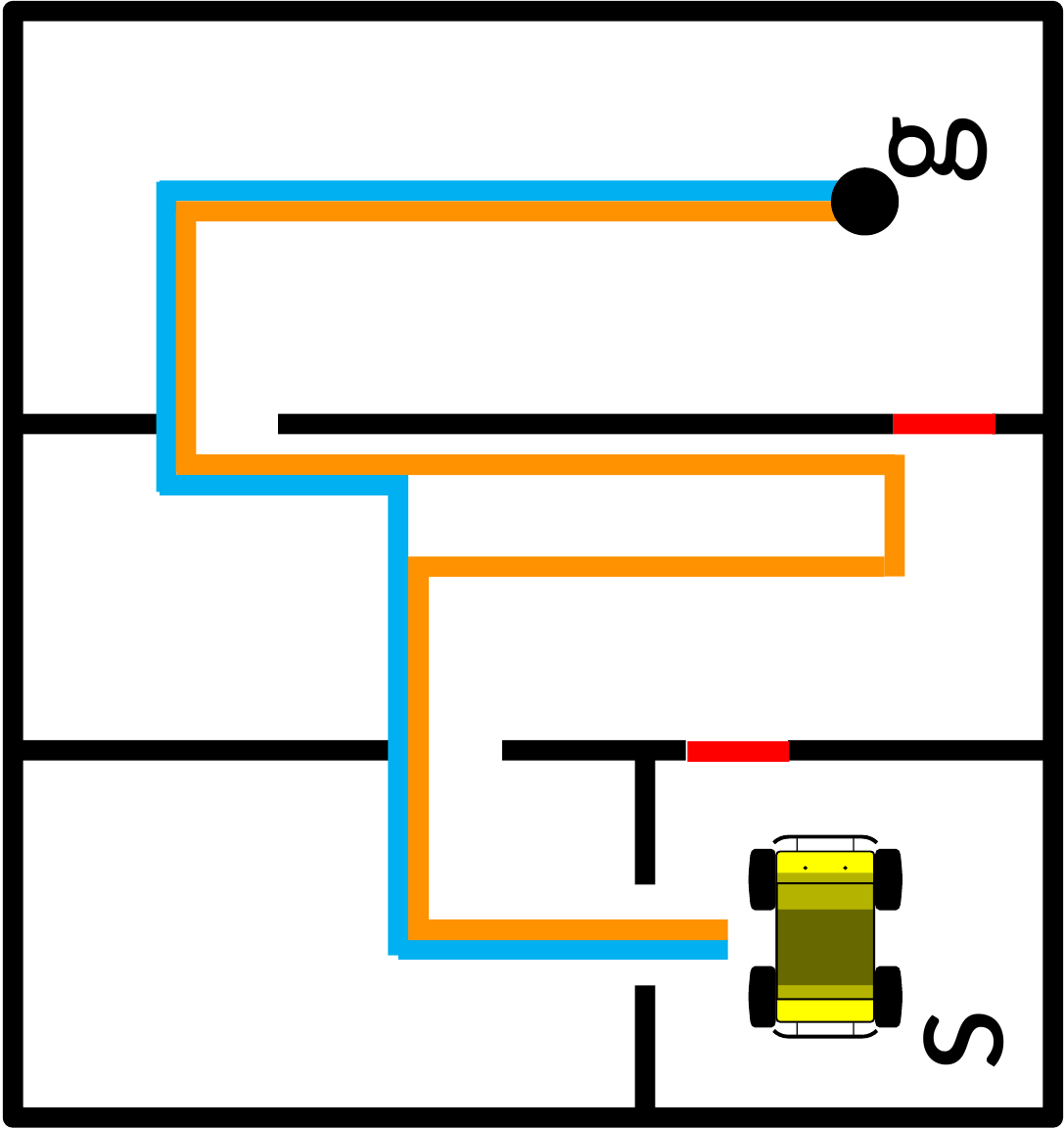}
        \end{center}
        \caption{$p=0.75$}
        \label{subfig:optbad1}
    \end{subfigure}
    \begin{subfigure}{.49\linewidth}
        \begin{center}
        \includegraphics[angle=-90,width=.95\linewidth]{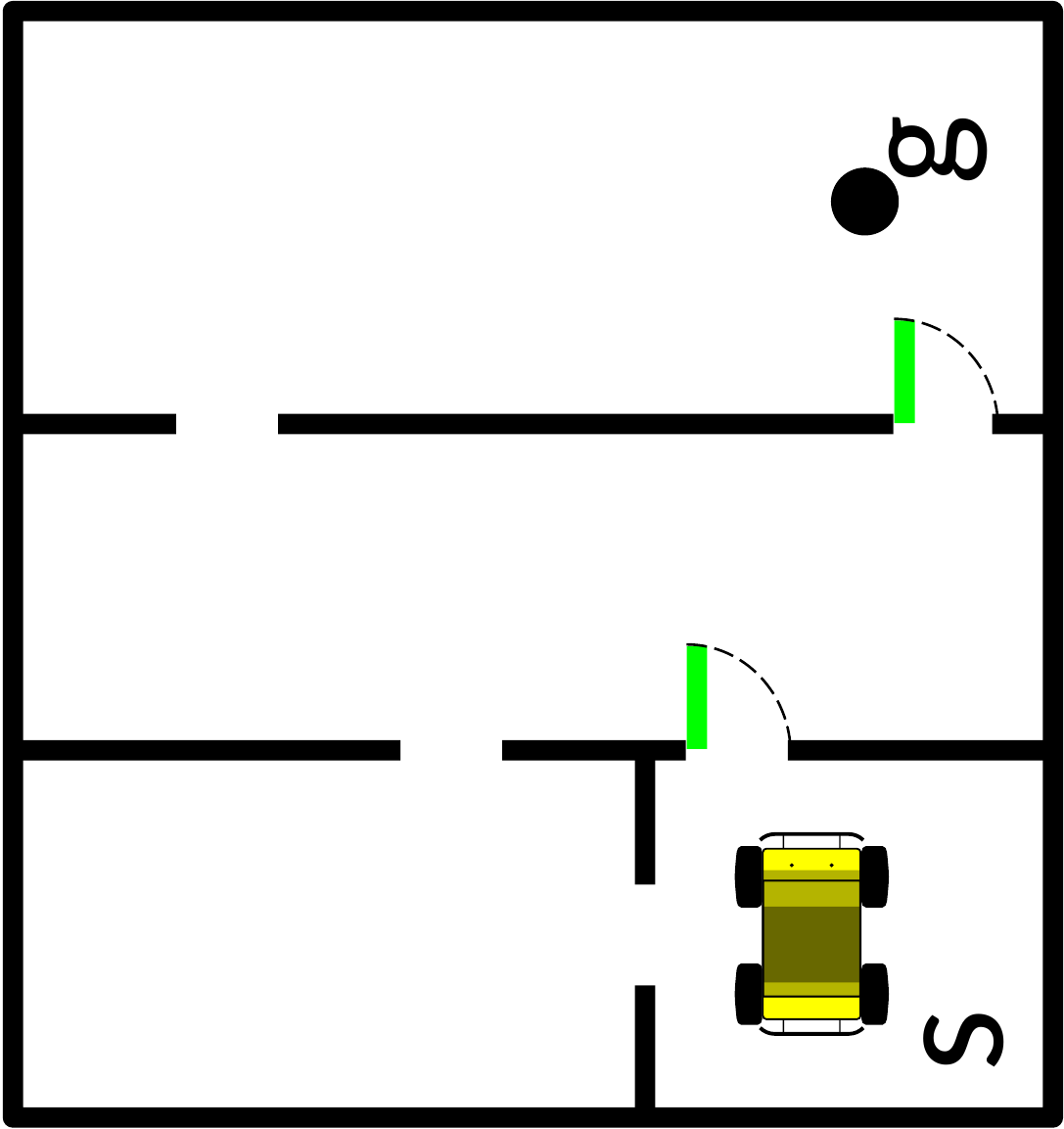}
        \end{center}
        \caption{$p=0.25$}
        \label{subfig:optbad2}
    \end{subfigure}
    \caption{Two configurations of the environment are shown in (b) and (c) with their probability of occurrence. In environment (b), the path a robot may take given an optimistic planner is shown in orange, while the optimal path is shown in blue.}
    \label{fig:optbad}
\end{figure}

\emph{Related Work:} The optimistic policy has been the most prevalent navigation approach because of its simplicity to implement. Examples include D*Lite \cite{Koenig2002}, LPA* \cite{koenig2004lifelong}, etc. More sophisticated algorithms such as Risk-Aware Graph Search~\cite{chung2019} aim to improve performance in unstructured environments.
However, as shown by the example in Figure \ref{fig:optbad}, sometimes there may be structure in the traversability of the environment that can be exploited. Kucner et al. ~\cite{kucner2013conditional} presents a method for learning motion patterns over time, but this work focuses more on dynamic uncertainties such as moving obstacles as opposed to the changes in traversability patterns considered here. 

A closely related problem to navigation with uncertainties is the Canadian Traveller's Problem (CTP). Informally, the CTP is the problem of, given a graph, finding the optimal path for a robot to take from start to goal. However, some of the edges are blocked (the robot cannot traverse that edge), and the state of an edge is only revealed to the robot when it reaches an endpoint of the edge. A common variant of this problem is the stochastic CTP, where each edge has a known probability of being blocked. The original problem was defined by Papadimitriou et al. \cite{papadimitriou1991shortest}, and  many other variations have been suggested (see \cite{Bnaya2009}, \cite{Lim2017}, \cite{Guo2019}).

Optimal policies for the stochastic CTP can be calculated using algorithms such as CAO*\cite{aksakalli2016based} or value iteration. However, they must limit the size of the problem or the number of observations the robot is allowed to make to keep the problem tractable.
Some algorithms for the stochastic CTP, such as UCT-CTP \cite{Eyerich2010}, optimize a policy by sampling environments based on the edge probabilities and estimating the expected cost by simulating the policy over this series of environments. These runs are known as \emph{rollouts}. The more rollouts that are performed, the more accurate the cost estimate will be, and generally the better the resulting policy at the cost of higher runtimes. Other algorithms use a heuristic to approximate the next step in a policy, such as the Hedged Shortest Path under Determinization (HSPD) algorithm proposed in \cite{Lim2017} or the algorithm proposed for the Reactive Planning Problem \cite{MacDonald2018}. Both assume there is structure to the environment uncertainties that can be exploited to find a minimum cost policy.

There has been very little work on two fronts. First, existing CTP variations and subsequently, their solutions, assume the probability of an edge being blocked is known, which is generally not true in practice. Only the work by Nardi and Stachniss~\cite{nardi2020long}, and Tsang et al. \cite{Tsang2019} addresses the problem of capturing this information from the environment. In~\cite{nardi2020long} authors propose a factor graph method to approximate the correlation between edge traversabilities. The proposed structure of the factor graph limits their approach to only capture binary correlations between edges. In comparison, our proposed method stores the observed environments during the task executions as subgraphs of the complete environment, which allows it to learn higher-order correlations between edges.

The second shortcoming of the existing literature is that there is little, if any, work that combines these CTP algorithms with existing navigation systems to form a pipeline that can be implemented on a physical platform. Many of these algorithms are graph-based, which need to be adapted to fit the continuous world the robot operates in. Additionally, it is typically assumed that the edge state can be determined at an edge endpoint, which is not necessarily true, especially in indoor scenarios. For example, a robot may not realize a door on the side of a hallway is closed or locked until it gets closer to the door.
Therefore, in this paper, we propose an edge resolver algorithm (i.e. determining whether an edge is blocked or unblocked) under uncertainties using a costmap \cite{Lu2014}.


 Aksakalli et al. \cite{aksakalli2016based} showed that the stochastic CTP can be modelled as a Markov Decision Process (MDP) and a deterministic Partially Observable Markov Decision Process (POMDP), both of which are formulations frequently used in reinforcement learning. However, generic POMDP solvers cannot be used to find optimal policies for the CTP in practical applications because the state space for the stochastic CTP is exponential.

There have been advances to utilize neural networks to approximate value functions for MDPs with large state and action spaces. This combination of neural networks and traditional reinforcement learning is coined Deep Reinforcement Learning (DRL). Kanezaki et al. \cite{Kanezaki2018} proposed a path planning algorithm based on DRL that performs well in the presence of many unexpected obstacles, but the behaviour is that of an optimistic algorithm. Furthermore, Lillicrap and Hunt et al. \cite{lillicrap2015continuous} introduced the Deep Deterministic Policy Gradient (DDPG), an algorithm to solve reinforcement learning problems with continuous action spaces (for example, balancing a quadruped, running, etc). Deep learning is also applied to predict an environment given a certain structure (\cite{saroya2020online}, \cite{caley2019deep}, \cite{shrestha2019}). The drawback of using these techniques for navigation with uncertainty is the need for training before the algorithm can be used in practice. These predictive approaches are trained for a specific type of environment (e.g. tunnels, offices), and can only be applied to these specific scenarios. Unlike tasks such as learning to play a specific video game or learning how to walk, the state space for the CTP can vary significantly with each new environment.
If the deployment environment is different from the training data, the algorithm may perform poorly and have to be retrained. Obtaining data for training such an algorithm is also a challenge. In contrast, our solution does not require any training, and can be deployed with minimal setup in a variety of scenarios.

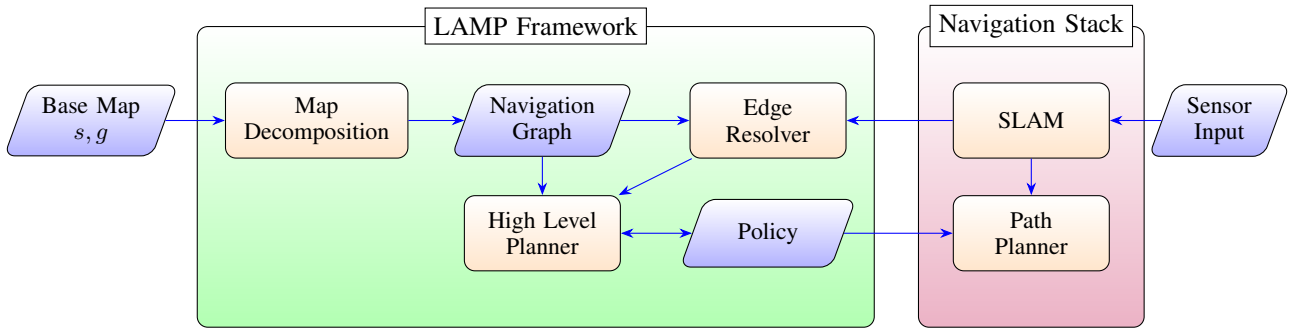
\begin{figure*}[t]
    \centering
    \begin{tikzpicture}[node distance = 1.2cm, nodes= {align = center}, >=Stealth]
    \tikzstyle{every node}=[font=\small]
        \node [input] at (-8,0) (base) {Base Map \\ $s,g$};
        \node [input] at (-2,0) (navgraph) {Navigation \\ Graph};
        \node [process, text width=6em] at (-5,0) (decomp) {Map \\ Decomposition};
        \node [process] at (1,0) (monitor) {Edge \\ Resolver};
        \node [process] at (-2,-1.5) (solver) {High Level \\ Planner};
        \node [input, minimum height = 0.9cm] at (1,-1.5) (policy) {Policy};
        \node [process] at (4.5,0) (localization) {SLAM};
        \node [process] at (4.5,-1.5) (planner) {Path \\ Planner};
        \node [input] at (7,0) (sensors) {Sensor \\ Input};

         \draw[control, ->] (base) -> (decomp);
         \draw[control, ->] (decomp) -> (navgraph);
         \draw[control, ->] (navgraph) -> (monitor);
         \draw[control, ->] (navgraph) -> (solver);
         \draw[control, ->] (policy) -> (planner);
         \draw[control, ->] (localization) -> (monitor);
         \draw[control, ->] (monitor) -> (solver);
         \draw[control, ->] (localization) -> (planner);
         \draw[control, ->] (sensors) -> (localization);
         \draw[control, <->] (solver) -> (policy);

        \begin{scope}[on background layer]
            \node[fit =(decomp)(solver)(monitor)(policy)(navgraph), box, minimum height=4cm, minimum width=9cm, bottom color=green!30] (lamp){};
            \node at (lamp.north) [fill=white,draw,font=\fontsize{10}{0},text=black] 
            (lrpp-label) {LAMP Framework};
        \end{scope}
        \begin{scope}[on background layer]
            \node[fit =(planner)(localization), box, minimum width=3cm, minimum height=4cm, bottom color=purple!30] (nav){};
            \node at (nav.north) [fill=white,draw,font=\fontsize{10}{0},text=black] 
            (nav-label) {Navigation Stack};
        \end{scope}
    \end{tikzpicture}
    \caption{Overall LAMP architecture. Arrows represent data flow between components. The parallelograms represent inputs and outputs while rectangles are nodes that process data.}
    \label{fig:overall}
\end{figure*}

\emph{Contributions}: The contributions of this paper are four-fold: First, we formulate the Learned Reactive Planning Problem (LRPP) as the problem of repeatedly navigating in an unknown environment while minimizing the total expected cost. Then we prove that a special case of the problem with known probabilities of environment configurations is PSPACE-hard. Second, we provide the RPP-Hybrid policy, a method in discrete environments to learn correlations between obstacles from past executions and propose an adaptive planning algorithm, which improves its policy with every task. Third, we propose the LAMP framework for navigating in continuous environments, which utilizes a hybrid topological representation of the environment, and integrates our discrete planner with an off-the-shelf navigation stack. Finally, we evaluate the performance of the proposed algorithms in discrete and continuous environments with an extensive set of simulated and real-world experiments.

A preliminary version of this work appeared in the conference paper~\cite{Tsang2019}.  This initial work introduced the LRPP in discrete environments and provided simulation results in a grid world. The key additional contributions of this paper are the extension of the problem to continuous environments, the hybrid decomposition of the environment, and the LAMP framework to integrate our discrete planner and a navigation stack.  We also greatly expanded the evaluation:  we provide a comparison of the discrete algorithm with the state-of-the-art method proposed in~\cite{nardi2020long}, an extensive set of experiments evaluating the LAMP framework in high-fidelity Gazebo simulations, and results of the LAMP framework in physical experiments. These latter experiments also evaluate the robustness of the proposed method in presence of localization and mapping errors.  We have also revamped and improved the presentation of the discrete algorithm, and have now formally proved several key properties.

An outline of the LAMP framework is shown in Figure~\ref{fig:overall}. The paper is organized as follows. In Section~\ref{sec:problem_formulation}, we formulate the problem and in Section~\ref{subsec:LRPPsol} we provide our solution approach in discrete environments, the RPP-Hybrid policy, which is used in the high level planner. In Section~\ref{sec:lamp} we provide the implementation of the solution approach for continuous environments, which forms the LAMP Framework component. We present and discuss our simulation and experimental results in Section~\ref{sec:implresults}.

\section{Problem Formulation}
\label{sec:problem_formulation}
In this section, we formulate the problem of executing a sequence of $T$ navigation tasks in an environment with uncertainties. 

\subsection{Environment Model} 

Consider a robot executing a sequence of $T \geq 1$ start-to-goal tasks. The environment is represented as an undirected graph $G = (V, E)$, where $V$ is the vertex set, $E$ is the set of edges between the vertices and let $c: E \rightarrow \mathbb{R}_{+}$ be a function assigning the cost of traversal to the edges. A path $P$ in the graph is defined by a sequence of vertices $\langle v_1, \ldots, v_k \rangle$ that satisfies $(v_i,v_{i+1}) \in E$ for all $i\in\mathbb{N}_{k-1}$. With a slight abuse of notation, we denote the cost of executing a path $P$ by $c(P) = \sum_{i=1}^{k-1} c(v_i,v_{i+1})$, and we denote the cost of the minimum cost path between $v,u \in V$ by $c(v,u)$. 

For each task $t \in \mathbb{N}_T$, there is a set of unblocked edges, which is a subset of $E$; edges not in this set are blocked and cannot be traversed. This set can be different for each task $t$. There are $m = 2^{|E|}$ edge subsets of $E$. Let the complete set of subgraphs of $G$ be denoted by $\mathcal{G} = \{G_1,\ldots, G_m\}$, where $G_i = (V, E_i)$ and $E_i \subseteq E$. Thus, for every task $t$, the robot is operating in a subgraph drawn from $\mathcal{G}$. The robot knows $G$, but it does not know which subgraph it is in.

Mathematically, the robot operates in a sequence of $T$ random graphs $\langle G_{X_1},\ldots,G_{X_{T}} \rangle$, where $ X_1,\ldots,X_T$ are independent and identically distributed random variables according to an \emph{unknown} probability mass function (pmf) $p_1,\ldots,p_m$.  That is, the probability that subgraph $t$ is $G_i$ is given by $\mathbb{P}(X_t = i) = p_i$.
We let $G_{X_t}$ denote the random subgraph for task $t$, and $G_{x_t}$ its realization.
When referring to the pmf, we drop the index and use random variable $X$.

We are interested in environments that exhibit a form of structured uncertainty in their traversability, which can be learned over time. By structured, we mean that absence or presence of particular obstacles for a given task $t$ exhibits some correlation (i.e., if the loading area of the warehouse is not traversable, then it is more likely that the packing area will also not be traversable).  In our model, this corresponds to the case where only a small subset of $\mathcal{G}$ dominates the unknown pmf.

\subsection{Robot Observation Model} \label{subsec:robotModel}

Consider a robot positioned at $v \in V$ that wishes to traverse $(v, u) \in E$. The robot is not aware of the current realization $G_x = (V, E_x)$ of the environment. Therefore, the robot must determine if the edge $(v, u)$ is traversable by observing the state (i.e. if it is blocked and thus not traversable). We assume that the robot is capable of sensing the edges incident to its current vertex $v$. Let $I_v \subseteq E$ be the set of edges incident to $v$, then the sensing action at $v$ is a mapping $\gamma_v: I_v \rightarrow \{\text{blocked},\text{unblocked}\}$ where $\gamma_v(e)=\text{unblocked}$ for $e\in E_{x}$ and $\gamma_v(e)=\text{blocked}$ otherwise. With this observation model, we define the \emph{robot state} as a tuple $(v, E^b, E^u)$ where $v$ is the vertex occupied by the robot and $E^b, E^u \subseteq E$ are the sets of observed blocked and unblocked edges, respectively. For simplicity, we assume that the robot does not incur a cost to observe the state of an edge, however, the cost of observation can be added without significant change to the problem formulation or solution.

\subsection{The Learned Reactive Planning Problem}
Given the environment setting and the robot observation model in the environment, the goal is to find a policy that outputs a sequence of actions for each realization of the environment and drives the robot towards the goal state. In more detail, a \emph{policy} is a mapping from each robot state $(v, E^b, E^u)$ to one of the following actions:  1) traverse an edge in $I_v$; 2) observe an edge in $I_v$; or 3) terminate. In this work, we are interested in a class of policies for the navigation problem, namely complete policies.

\begin{definition}[Complete Policy \cite{MacDonald2018}]
Consider a graph $G$ with start $v_s$ and goal $v_g$, along with its complete set of subgraphs $\mathcal{G}$. A policy $\pi$ is complete if for every subgraph $G_i \in \mathcal{G}$ it produces a finite sequence of actions that either reaches the goal state, or correctly determines that the goal cannot be reached in $G_i$.
\end{definition}

Given a graph $G_i$, let $A^{\pi}_{G_i} = \langle a_1, \ldots, a_l\rangle$ be the sequence of actions produced by a policy $\pi$. With a slight abuse of notation, we denote the cost of an action $a$ with $c_a(a)$ where the observation and termination actions have zero costs and the cost of a traversal action across $(u, v)$ is given by $c(u, v)$. Then the total cost of executing $A^{\pi}_{G_i}$ is $\mathrm{cost}(A^{\pi}_{G_i}) = \sum_{i = 1}^{l} c_a(a_i)$. Therefore, the expected cost to complete a task under policy $\pi$ is 
\[
\mathbb{E}_{X}\big[\text{cost}(\pi)] = \sum_{i \in \mathbb{N}_m} \mathbb{P}(X = i) \mathrm{cost}(A^{\pi}_{G_i}).
\]

We consider a sequence of $T$ tasks where the goal is to minimize the total expected cost of executing the tasks. In the execution of task $t$, the robot may use the information collected during tasks $1, \ldots, t - 1$ on the traversability of the different regions of the environment. Formally, we define this as the Learned Reactive Planning Problem.

\begin{problem}[Learned Reactive Planning Problem (LRPP)]
Given a graph $G$ with unknown pmf over all subgraphs $\mathcal{G}$, a start and goal $v_s, v_g\in V$ and number of tasks $T$, find a sequence of $T$ complete policies, $\langle \pi_1,\ldots,\pi_T \rangle$, that minimizes $\sum_{t=1}^T \mathbb{E}_{X}\big[\text{cost}(\pi_t)]$, where $\pi_t$ may depend on the observations made in tasks $1,\ldots, t-1$.
\end{problem}

The following result provides some insight into the complexity of the Learned Reactive Planning Problem.

\begin{proposition}
    Suppose that the robot has completed $T-1$ tasks, which is sufficiently large to have learned a function that can return the probability of any subgraph in polynomial time.  Then, the problem of computing the policy $\pi_T$ that minimizes $\mathbb{E}_{X_T}\big[\text{cost}(\pi_T)]$ is PSPACE-hard.
\end{proposition}

\begin{proof}
Consider an instance of the stochastic Canadian Travelers problem (CTP).  This consists of a graph $G_{\text{CTP}} = (V,E)$, a cost on each edge $c_{\text{CTP}}: E\to \mathbb{R}_{>0}$, and a probability for each edge $p:E \to [0,1]$, giving the probability $p(e)$ that the edge $e\in E$ is unblocked. The goal is to find a policy that minimizes the expected cost from start to goal.  To reduce this problem to the $T$th task in the LRPP, we create the following instance with a known pmf: We set $G = G_{CTP}$, $c = c_{\text{CTP}}$, and we let $\bar p(i)$ denote the learned function that for each subgraph $G_i = (V, E_i) \in \mathcal{G}$, returns its probability $\mathbb{P}(X_t = i) = p_i$ as
\[
\bar p(i) = p_i = \prod_{e\in E_i} p(e) \prod_{e\in E\setminus E_i} \big(1- p(e)\big).
\]
For each $i$, this function returns a probability in polynomial time, as it requires computing the product of $|E|$ numbers.  Thus, the reduction from CTP to task $T$ of the LRPP with a learned function can be performed in polynomial time.

An optimal policy $\pi_T$ for task $T$ minimizes $\mathbb{E}_{X_T}\big[\text{cost}(\pi_T)]$, and thus this policy is also optimal for the stochastic CTP.  Since the stochastic CTP is PSPACE-hard~\cite{Fried2013}, the problem of computing the optimal policy for the $T$th task is also PSPACE-hard.
\end{proof}

Note, the above result characterizes the complexity of only one task in the LRPP, and only in the case when we have access to a function that can return the probability of any given subgraph.  For problems with an unknown input such as the general LRPP, notions of complexity are considerably more subtle; see for example advice complexity~\cite{arora2009computational} or the complexity of trial-and-error~\cite{ivanyos2018complexity}.  However, we believe the above result is sufficient to demonstrate the complexity of the proposed problem.

\subsection{A Special Case of LRPP}

A special case of the LRPP was formally posed by MacDonald et al. \cite{MacDonald2018} as the Reactive Planning Problem (RPP) where there is a single navigation task (i.e., $T = 1$) and the pmf over $X$ is known.  It is assumed that only a small subset $\mathcal{G}'\subset \mathcal{G}$ of the subgraphs have non-zero probability.

The proposed policy in~\cite{MacDonald2018} for the RPP tries to balance the act of taking observations (known as exploration) and moving towards the goal (known as exploitation). At each step, the policy finds the next observation vertex that minimizes a product of the expected entropy of taking an observation and a linear combination of the travel cost to the observation vertex, the cost of observation and the expected cost from the observation vertex to the goal. The proposed policy is complete and is represented as a decision tree where the size of the representation is polynomial in the number of the subgraphs in $\mathcal{G}'$ and the number of vertices in $V$. Authors in~\cite{MacDonald2018} show that a complete policy exists for all start-to-goal navigation tasks in undirected graphs.

For more details on the RPP and the solution outlined in this section, refer to \cite{MacDonald2018}. For the purposes of this paper, the reader only needs to understand that by providing 1) an undirected graph $G$, 2) the start and goal vertices, $v_s$ and $v_g$, 3) a set of subgraphs $\mathcal{G}'$, and 4) a corresponding pmf over the subgraphs, a solution to the RPP will return a complete policy for $\mathcal{G}'$.

\section{LRPP Solution Approach}
\label{subsec:LRPPsol}

There are three key challenges to address when considering the approach in \cite{MacDonald2018} to solve the LRPP. First, in the LRPP the pmf over the subgraphs is unknown. To address this, in Section~\ref{subsubsec:mapfilter},  we propose a Map Memory Filter to add a new map $M_t$ to the set of collected maps $\mathcal{M}_t$ and estimate the pmf. Second, while executing a task, the robot might encounter an environment that it has not experienced before. To handle such situations, we introduce the idea of switching to an optimistic policy in Section~\ref{subsubsec:lambda}. Finally, the goal is to minimize the expected cost of executing a sequence of tasks by learning a model of the environment. In Section~\ref{subsubsec:policyupd} we explain how the policy generating algorithm proposed by \cite{MacDonald2018} can be utilized to generate and update a complete policy $\pi_t$ based on changes to the estimated pmf. Figure~\ref{fig:lrpp_flow_chart} summarizes our proposed solution to the LRPP.

\begin{figure}
    \centering
    \begin{tikzpicture}[nodes= {align = center}, >=Stealth]
        \node [root, scale=0.8] (s) at (0,2) {Start\\task};
        \node [decision, scale=0.8,inner xsep=0] (1) at (3,2) {Observations\\Consistent\\with $\mathcal{M}_{t-1}$};
        \node [process, scale=0.8,inner xsep=0] (2) at (5,0.6) {Switch to\\optimistic};
        \node [finish, scale=0.8] (g) at (6,2) {Terminate};
        
        \node [process, scale=0.8] at (4.5,-0.5) (mapfilter) {Filter\\Maps};
        \node [process, scale=0.8] at (1.5,-0.5) (buildpol) {Build RPP\\Policy};
        
        \draw [thick, ->] (s) -- node[near start,above,font=\scriptsize]{$\pi_t$} (1);
        \draw [thick, ->] (1) -- node[near start,above,font=\scriptsize]{Yes} (g);
        \draw [thick, ->] (1.south) |- node[near start, right, font=\scriptsize]{No} (2);
        \draw [thick, ->] (2) -| (g.south);
        \draw [thick, ->] (g.east) -- (7,2) |- node[near end,above,font=\scriptsize]{$M_t$} (mapfilter);
        \draw [thick, ->] (mapfilter) -> node[midway,above,font=\scriptsize]{$\mathcal{M}_t$} (buildpol);
        \draw [thick, ->] (buildpol)  -| node[near end,right,font=\scriptsize]{$\pi_{t+1}$} (s);
    \end{tikzpicture}
    \caption{Flow chart of the proposed algorithm for the LRPP}
    \label{fig:lrpp_flow_chart}
\end{figure}
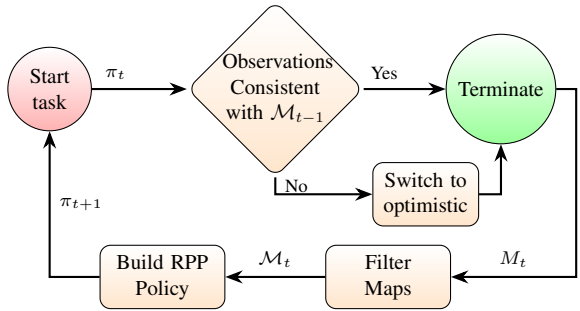

\subsection{Map Memory Filter} \label{subsubsec:mapfilter}
After the robot performs $n$ actions within the environment, let $E_{t,n}\subseteq E$ denote the set of edges with a known state. We define the robot's understanding, or \emph{map} of $G_{x_t}$, after its $n$th action, as the tuple $M_{t,n} = (E_{t,n}^{\text{b}},E_{t,n}^{\text{u}})$ with known blocked edges $E_{t,n}^{\text{b}} = \{e\in E_{t,n}| e\not\in E_{x_t}\}$ and known unblocked $E_{t,n}^{\text{u}} = \{e\in E_{t,n}| e\in E_{x_t}\}$. Note that $E_{t,n}^{\text{b}}$ and $E_{t,n}^{\text{u}}$ form a partition of $E_{t,n}$. When the task is finished, the robot stores the map in the list $\mathcal{M}_t = [M_1,\ldots,M_t]$, where $n$ is removed to indicate the task is completed.

With the map defined, we introduce a filter to reduce the amount of map information stored in memory and subsequently the search space of the policy construction algorithm. We only store maps that do not \emph{agree}, which means that each stored map corresponds to a different realization of $G$. A formal definition of the map agreement is given below.

\begin{definition}[Map Agreement]
Given maps $M_1$ and $M_2$, we say $M_2$ agrees with $M_1$ if $E_2^{\text{b}}\cap E_1^{\text{u}}=\emptyset$ and $E_2^{\text{u}}\cap E_1^{\text{b}}= \emptyset$.
\label{def:agreement}
\end{definition}
Figure~\ref{fig:mapEx} shows three maps obtained while executing tasks in different realizations. Solid edges represent edges that are observed and traversable, dashed edges represent unobserved edges and the lack of an edge represents an observed and blocked edge. Map (a) \emph{agrees} with Map (c), but not Map (b) because edge $(1,2)$ is unblocked in Maps (a) and (c) and blocked in Map (b).

Note that the robot does not need to know the whole environment to accomplish its task and can leave regions unmapped, which may result in different realizations seeming identical to the robot. 
\begin{figure}
    \begin{subfigure}{.32\linewidth}
        \begin{center}
            \begin{tikzpicture}
            \begin{scope}[every node/.style={circle,draw}]
        	\node (1) [align=center] at (-.5,-.5) {$1$};
            \node (2) [align=center] at (0,.5) {$2$};
            \node (3) [align=center] at (.5,-.5) {$3$};
            \end{scope}
            \begin{scope}[every node/.style={circle},
                          every edge/.style={draw=black}]
                \path [-,every node/.style={sloped,anchor=south,auto=false}] (1) edge[bend left=0] node[above] {} (2);
                \path [-,every node/.style={sloped,anchor=south,auto=false}] (2) edge[bend left=0,dashed] node[above] {} (3);
                \path [-,every node/.style={sloped,anchor=south,auto=false}] (3) edge[bend left=0,dashed] node[below] {} (1);
            \end{scope}
    \end{tikzpicture}
        \end{center}
    \caption{}
    \label{fig:mapEx_a}
\end{subfigure}
\begin{subfigure}{.32\linewidth}
        \begin{center}
            \begin{tikzpicture}
        \begin{scope}[every node/.style={circle,draw}]
        	\node (1) [align=center] at (-.5,-.5) {$1$};
            \node (2) [align=center] at (0,.5) {$2$};
            \node (3) [align=center] at (.5,-.5) {$3$};
        \end{scope}
        \begin{scope}[every node/.style={circle},
                      every edge/.style={draw=black}]
            \path [-,every node/.style={sloped,anchor=south,auto=false}] (2) edge[bend left=0] node[above] {} (3);
            \path [-,every node/.style={sloped,anchor=south,auto=false}] (3) edge[bend left=0] node[below] {} (1);
        \end{scope}
    \end{tikzpicture}
        \end{center}
        
    \caption{}
    \label{fig:mapEx_b}
\end{subfigure}
\begin{subfigure}{.32\linewidth}
    \begin{center}
        \begin{tikzpicture}
    \begin{scope}[every node/.style={circle,draw}]
        	\node (1) [align=center] at (-.5,-.5) {$1$};
            \node (2) [align=center] at (0,.5) {$2$};
            \node (3) [align=center] at (.5,-.5) {$3$};
    \end{scope}
    \begin{scope}[every node/.style={circle},
                  every edge/.style={draw=black}]
        \path [-,every node/.style={sloped,anchor=south,auto=false}] (1) edge[bend left=0] node[above] {} (2);
        \path [-,every node/.style={sloped,anchor=south,auto=false}] (2) edge[bend left=0] node[above] {} (3);
        \path [-,every node/.style={sloped,anchor=south,auto=false}] (3) edge[bend left=0] node[below] {} (1);
    \end{scope}
    \end{tikzpicture}
    \end{center}
    \caption{}
    \label{fig:mapEx_c}
\end{subfigure}
    \caption{Example of map agreement}
    \label{fig:mapEx}
\end{figure}
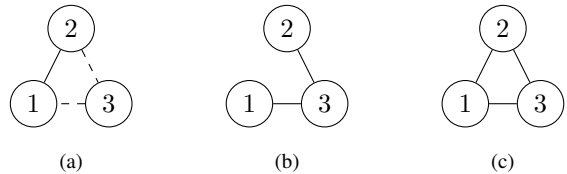

To minimize the number of maps stored, we introduce a compact representation of maps that agree with each other, which we call \emph{super maps}.

\begin{definition}[Super Maps]
A map $M_j$ with $j\in \mathbb{N}_t$ is a super map if all $M_i$ for $j\neq i\in\mathbb{N}_t$ that agree with $M_j$ satisfy $E_i^{\text{b}}\subseteq E_j^{\text{b}}$ and $E_i^{\text{u}}\subseteq E_j^{\text{u}}$.
\end{definition}

 The problem of computing a minimal set of super maps can be formalized as follows.

\begin{problem}[Map Merging Problem]
    Given a set of collected maps from each task, $\mathcal{M}_T = [M_1, M_2, \ldots, M_T]$, find a minimum partition of $\mathcal{M}_T$ such that every map in each subset agree with each other.
\end{problem}

We define merging a set of maps that are in agreement with each other as the union of the blocked and unblocked edges in the maps. Formally, let $M_i$ and $M_j$ be two agreeing maps, then the map obtained from merging them is $(E^b_i \cup E^b_j, E^u_i \cup E^u_j)$. Observe that merging the maps in a subset forms a super map, and thus the solution to the Map Merging Problem provides a compressed representation of the robot's past experiences. We redefine $\mathcal{M}_t$ as the set of \emph{super maps} at the end of task $t$.

Prior to providing our algorithm for the Map Merging Problem, we show that this problem is NP-hard by a reduction from the Minimum Graph Coloring (MGC) problem. Given a graph, the MGC problem is the problem of assigning a color to each vertex such that no two incident vertices have the same color and the total number of colors is minimized. The MGC problem is a well-known NP-hard problem~\cite{karp1972reducibility}.

\begin{lemma}
The Map Merging Problem is NP-hard.
\end{lemma}
\begin{proof}
Consider a graph $G = (V, E)$ as an instance of the MGC problem. For each vertex $v \in V$, we construct a map $M_v$ such that if $(u, v) \in E$ then $M_u$ and $M_v$ do not agree with each other. The construction of these maps is given below.
Consider a complete graph $G'  = (V, E')$ as the base map, then we have:
\begin{enumerate}
    \item For all $u \in V$ and for all $w, v \not= u$, the edge $(w, v) \in E'$ is an unknown edge in the map $M_u$,
    \item for all $(u, v) \not \in E$, the edge $(u, v)$ is unblocked in both $M_u$ and $M_v$; and finally
    \item for all $(u, v) \in E$, the edge $(u, v)$ is blocked in $M_u$ and is unblocked in $M_v$.
\end{enumerate}
Let $S = \{S_1, \ldots, S_k\}$ be the minimum size partition of the constructed Map Merging problem such that the maps in the subsets agree with each other. By the construction of the instance, two maps $M_u$ and $M_v$ agree with each other if and only if $(u, v) \not \in E$. Hence, for all $M_u, M_v \in S_i$, the edge $(u, v) \not \in E$. Therefore, the corresponding vertices in each subset do not share any edge and these vertices can be colored with the same color. Therefore, the graph $G$ can be colored with $k$ colors. Now suppose there exists a coloring of  $G$ with $k' < k$ colors, then by the construction of the Map Merging Problem instance, there is a partition $S'$ of maps with size $k'$ such that the maps in each subset of the partition agree with each other. This is a contradiction.   
\end{proof}

Observe that the maps in the set $\mathcal{M}_T$ are added sequentially with each task execution and so the online map merging problem is a form of online graph coloring~\cite{Gramm2009}.  Algorithm~\ref{algo:mapFilter} details our map merging method, which is a greedy approach that immediately adds a map to a subset via merging if it agrees with an existing super map. It is analogous to the First Fit approach to the online graph coloring problem, which, while not an approximation algorithm~\cite{Vishwanathan1992}, provides good performance in practice. The algorithm takes the new map $M_t$ and the previous super maps $\mathcal{M}_{t-1}$ as inputs. The function $\textsc{AgreeingMaps}$ finds the subset of super maps in $\mathcal{M}_{t-1}$ that \emph{agree} with $M_t$ according to Definition~\ref{def:agreement}. If there is no such super map, then we add the new map to the set of super maps (Line 3). If the set of agreeing super maps $S$ is not empty, then the function \textsc{BestMapToMerge} finds the best super map in $S$ to merge $M_t$ with. The simplest definition of this function is to merge $M_t$ with the first agreeing map in $S$. A more sophisticated version of this function is to find the super map that adds the minimum number of additional blocked edges to the super map upon merging, i.e.,
\[
\argmin_{M_j \in S} |E^b_j \cup E^b_t| - | E^b_j|. 
\]  
If there are multiple minimizers, break the tie by randomly selecting one of them. Observe that large numbers of blocked edges in the super maps may result in conservative policies by the robot in future task executions. Therefore, the method proposed in \textsc{BestMapToMerge} favors exploration by greedily minimizing the number of blocked edges in the collected super maps.

\begin{algorithm}
\KwIn{$M_t$, $\mathcal{M}_{t-1}$}
\KwOut{$\mathcal{M}_{t}$}
$S \leftarrow $\textsc{AgreeingMaps}($\mathcal{M}_{t-1}$, $M_t$) \Comment{$S \subseteq \mathcal{M}_{t-1}$}\\
\If{$S = \emptyset$}
{
return $\mathcal{M}_{t-1}\cup M_t$\\
}
$M_j \leftarrow $\textsc{BestMapToMerge}($S$, $M_t$) \Comment{$M_j \in S$}\\
$M_j \leftarrow $\textsc{MergeMaps}($M_j$, $M_t$)\\
return $\mathcal{M}_{t-1}$
 \caption{\textsc{MapFilter}}
 \label{algo:mapFilter}
\end{algorithm}

Using this method of storage, the expected cost estimate of each policy can be calculated as
\begin{equation}
\label{eq:eCostSimpleMaps}
\mathbb{E}_X[\text{cost}(\pi_t)] = \sum_{M_j\in\mathcal{M}_t} \bigg(\frac{n_j + 1}{t}\bigg)\mathrm{cost}_{\pi_t}(M_j)\>,
\end{equation}
where $n_j$ is the number of maps the robot has experienced by task $t$ that have been merged with super map $M_j$ and $\text{cost}_{\pi_t}(M_j)$ is the cost of executing policy $\pi_t$ in super map $M_j$. Thus the estimated probability of encountering $M_j$ in the next task is $\hat{p}_{M_j} = (n_j + 1)/t$. Let $\hat{P}=[\hat{p}_{M_0}, \hat{p}_{M_1}, \ldots ,\hat{p}_{M_t}]$ where $M_j \in \mathcal{M}_t$ for all $j$ in $\mathbb{N}_{T}$, forming our estimate of the pmf of $\mathcal{G}$. Note that $\mathcal{M}_0=\{(\emptyset,E)\}$ and initialize $n_0=0$, leading to the robot's initial assumption that all edges in $G$ are unblocked.

\subsection{Switching to an Optimistic Policy} 
\label{subsubsec:lambda}
Recall that an optimistic policy plans the shortest possible path from the robot's current location to the goal given the known map, assuming all unknown edges are unblocked. If the move cannot be completed because of an unexpected blocked edge, then it recomputes and follows the shortest path in the updated map. Formally we define an optimistic policy as follows:

\begin{definition}[Optimistic Policy $\lambda$]
An optimistic policy $\lambda$ computes a sequence of move commands online to lead the robot to the goal state. Such a policy must be complete.
\end{definition}

For a given task $t$, the robot starts by following the preplanned paths in the RPP policy $\pi_t$ until either 1) an obstacle prevents the robot from continuing (in which case the robot is in a new map) or 2) all super maps that are consistent with the robot's observations have no path to the goal $v_g$. In either case, the robot switches to $\lambda$ to finish the task. This satisfies the complete policy requirement, and the policy is updated each time a new task is completed. The preplanned paths are expected to be more efficient at reaching $v_g$ than $\lambda$, and as such the probability of the robot switching to $\lambda$ should be minimized. This policy will be referred to as the RPP-Hybrid policy.

\subsection{Policy Construction and Update} 
\label{subsubsec:policyupd}
A policy for task $t$ can be constructed as a binary tree $\pi_t = (N,L)$ by using the super maps $\mathcal{M}_{t-1}$ and estimated pmf $\hat{P}_{t-1}$ collected thus far by the robot as inputs into the RPP problem. The nodes $N$ of the tree are given by tuples $(Y,v,e)$ for belief $Y=\{i \in \mathcal{M}_{t-1}| M_i \text{ agrees with } M_{t,n}\}$ at vertex $v \in V$. The edge $e$ is an \emph{observation} at vertex $v$. For each node, $L$ contains a path from the parent node to the current node. There are two possible outcomes for each observation, one corresponding to $e$ being unblocked, i.e., $e\in E_{x_t}$ and the other corresponding to $e$ being blocked, i.e., $e \not \in E_{x_t}$. If $e=\emptyset$, then either $v=v_g$ or there is no path to goal in any of the agreeing super maps, resulting in an incomplete policy. The policy can be made complete by augmenting it to allow the robot to switch to the optimistic policy for the remainder of the task.

To use $\mathcal{M}_t$ as an input to the RPP solution, it first needs to be converted into a set of edge subsets, where each subset represents the passable edges in each super map. Since each super map in $\mathcal{M}_{t}$ may only be a partial representation of a realization, it is necessary to make some assumptions to fill in missing information. If the state of an edge in super map $M_j$ is unknown (i.e., $e \notin E^{\text{b}}_j$ and $e \notin E^{\text{u}}_j$), we assume it to be unblocked. Formally, the edge subset for $M_j$ will be $E^{\text{u}}_j \cup (E\setminus (E^{\text{u}}_j\cup E^{\text{b}}_j))$. This choice encourages the robot to explore, as it will attempt to traverse an unknown edge if it is beneficial. 

\begin{figure}
  \centering
  \begin{subfigure}[t]{.49\linewidth}
    \centering
      \includegraphics[width=.7\linewidth]{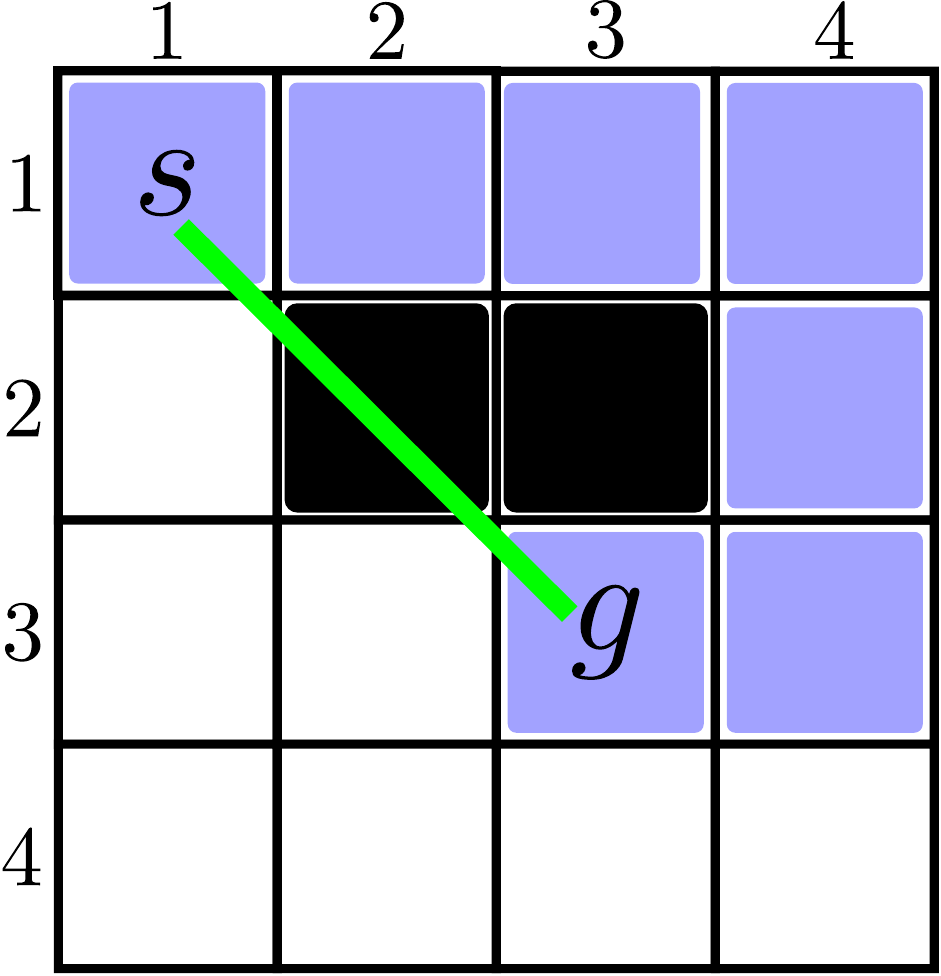}
      \caption{}
  \end{subfigure}
  \begin{subfigure}[t]{.49\linewidth}
    \centering
      \includegraphics[width=.7\linewidth]{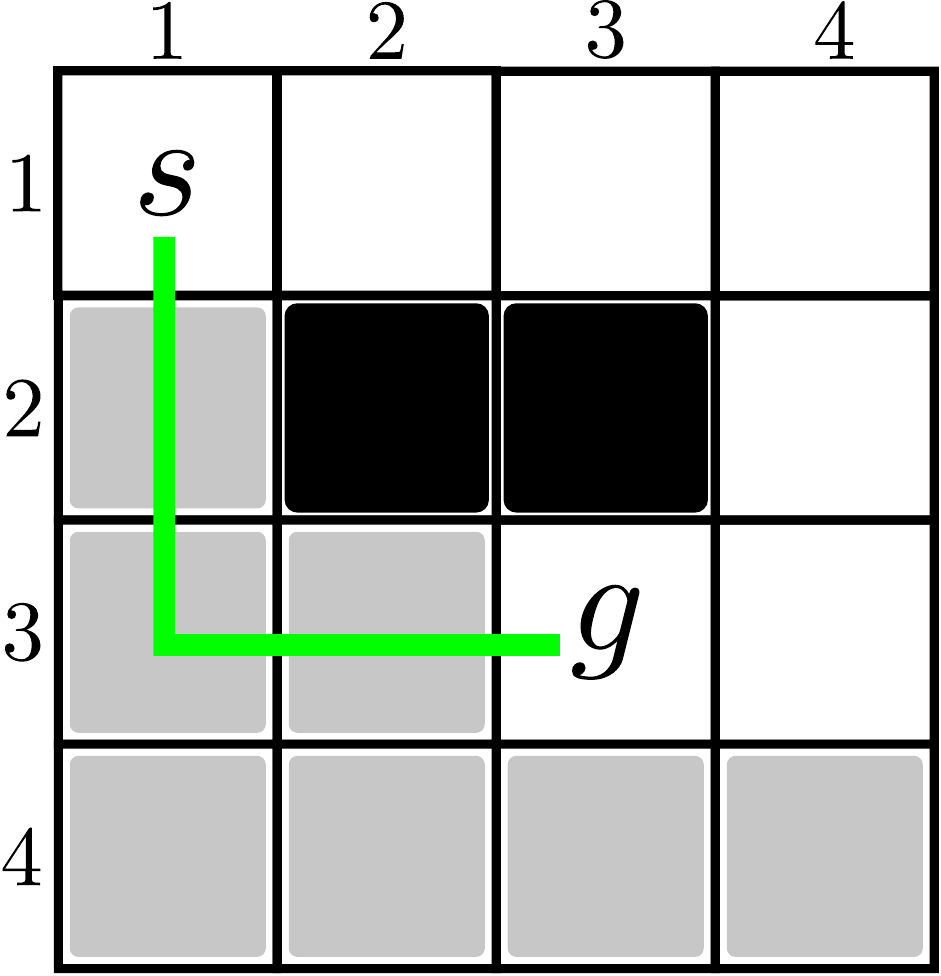}
      \caption{}
      \label{fig:example_map_b}
  \end{subfigure}
  \caption{(a) shows a realization for task $t$, and (b) is the collected $M_t$. The green line in (a) is the path determined by the policy $\pi_t$, in (b) by the policy $\pi_{t+1}$. The blue squares are the path that the robot actually took, the grey squares are unknown.}
  \label{fig:LRPPknownM}
\end{figure}

For example, consider the scenario in Figure~\ref{fig:LRPPknownM}. Assuming the robot only has an empty grid for $M_0$ in $\mathcal{M}_t$, it attempts to execute the task in (a) using the green path. However, it must switch to the optimistic policy, and takes the blue shaded path. The robot stores the map from this execution as $M_1$ (see Figure~\ref{fig:example_map_b}), where the state of the grey squares are unknown. When building the policy, an observation for the edge ((1,1),(2,2)) will be selected, and in the case this edge is blocked, a path will be calculated from $s$ to $g$ in $M_1$ since no other agreeing maps exist. If only the partial map $M_1$ was available, the blue shaded path would be used in the policy. However, since we are assuming the grey squares are unblocked, the algorithm will select the green path in (b). 
Even if that assumption was proven wrong during task execution it will result in more knowledge of the realization, and the next time the policy is built, the algorithm will not repeat the same path for that particular super map.

Finally, we present our entire solution in Algorithm~\ref{algo:updateRule}, which covers task execution and policy construction.
In Line~\ref{ln:initM}, we initialize the set of super maps $\mathcal{M}_t$ with $E$ as a set of unblocked edges. In other words, the robot is aware of all edges that it could potentially move across.  Such information could come from a floor plan of the environment, containing all permanent obstacles. The robot initially assumes that all edges are unblocked. This assumption ensures that the optimistic policy $\lambda$ will always initially attempt the shortest possible path to $v_g$.
In line 3, the policy $\pi_t$ is constructed by the RPP algorithm using the set of super maps $\mathcal{M}_t$, which contains the edge subsets, and the estimated pmf $\hat{P}$.
In lines $5-8$, the robot executes the task by following the policy $\pi_t$ constructed by the RPP algorithm until it reaches a terminal state, updating its set of super maps, along with the estimated pmf and edge subsets, in lines $9$ and $10$ before executing the task again.

\begin{algorithm}
  \KwIn{$G = (V,E)$, $v_s$, $v_g$}
  $\mathcal{M}_0 = [(E,\emptyset)]$\label{ln:initM}\\
 	\For{$t=1,\ldots,T$}{
 	  $\pi_t = \textsc{BuildPolicy}(G, \mathcal{M}_{t-1}$, $\hat{P}_{t-1}$,$v_s$,$v_g$)\label{ln:build}\\
 		Initialize state $\mathcal{R}_{t,n} = (v_s,\emptyset,\emptyset)$ for $n=0$\;
\While{$\mathcal{R}_{t,n}\;\text{\upshape not terminal}$}{
			Execute $\pi_t(\mathcal{R}_{t,n})$\Comment{{if switched policies, wait until $\lambda$ terminates}}\\ 
			Update $\mathcal{R}_{t,n}$\\
			Increment $n$\\
		}
		$M_t= (E_{t,n}^{\text{b}},E_{t,n}^{\text{u}})$ from $\mathcal{R}_{t,n}$\\
		$\mathcal{M}_t$ = \textsc{MapFilter}($M_t$,$\mathcal{M}_{t-1}$)\\
		Update $\hat{P}_t$\\
 	}

\caption{\textsc{SequentialTaskCompletion}}
\label{algo:updateRule}
\end{algorithm}

\section{LAMP -- from Theory to Application }
\label{sec:lamp}
In this section, we discuss our approach in applying the solution outlined in Section~\ref{subsec:LRPPsol} to realistic scenarios. First, we discuss how the robot interprets its environment as a graph, given a floor plan. Second, we introduce an edge resolving algorithm for determining whether an edge is traversable or not, given an updated costmap. Third, we discuss how the robot executes the RPP-Hybrid policy in a real-world setting where some of the assumptions from the LRPP no longer apply.

\subsection{High Level Idea} \label{sec:lampfwk}

The architecture of our solution, the Learn a Motion Policy (LAMP) framework is shown in Figure~\ref{fig:overall}. The framework is meant to be added to a standard navigation stack. Although the navigation stack in Figure~\ref{fig:overall} is based on the ROS navigation stack~\cite{Marder-Eppstein2010}, others have a similar structure~\cite{Huskic2019}.

The edge observer takes the costmap and the estimated robot's pose from the localization system to resolve edge states (i.e. determine if they are blocked or unblocked). The edge states are given to the high level planner, which uses it to build maps of the environment. The high level planner then constructs a policy between tasks and implements them during task execution by instructing the path planner on which vertex to go to next. 

The local planner is responsible for planning and executing a kinematically feasible path from the robot's current position to the next vertex dictated by the high level planner. There are a myriad of solutions available for these lower level components of a standard navigation stack, and their solutions will not be discussed in this paper.

\subsection{Interpreting the Environment} \label{sec:mapframework}
Since we are concerned with blocked edges that may significantly alter the route the robot may take, we propose converting a base occupancy grid (taken from a scaled floor plan where only permanent obstacles are marked) into a hybrid topological-metric map. A hybrid map combines a topological graph and an occupancy grid. For more details on the reasoning behind this choice of map structure, please refer to Section 5.2.1 of the thesis written by Tsang \cite{Tsang2020}. Constructing a topological graph from an occupancy grid is not a new idea; algorithms have been proposed by Thrun \cite{Thrun1998}, Liu et al. \cite{Liu2014} and Bl\"{o}chliger et al. \cite{Blochliger2017}.  While we will not be proposing a new algorithm for converting an occupancy grid into a graph, we will briefly cover the process and borrow some terminology from \cite{Blochliger2017}.

The base occupancy grid is decomposed into regions of free space, which we will refer to as \emph{submaps}. For example in an indoor environment, rooms and hallways would become regions. If rooms or hallways are divided into multiple regions, our approach is still valid. In Figure~\ref{fig:navgraph}, these submaps are represented by different coloured areas. The submap boundaries where the robot can cross from one region to another are referred to as \emph{portals} (or \emph{critical lines} in \cite{Thrun1998, Liu2014}), represented by the white rectangles in Figure~\ref{fig:navgraph}. The topological graph is referred to as a \emph{navigation graph} in this work. The navigation graph is a weighted, undirected graph and is defined by $G = (V,E)$ with a cost on each edge $c: E \rightarrow \mathbb{R}_{\geq 0}$. The vertices and edges are defined by the following:
\begin{itemize}
        \item \textbf{Vertices} are the center points of each portal. Thus, each $v \in V$ represents a point $v \in \mathbb{R}^2$.
        Additional vertices that represent an area of interest (ex. start, goal) can be added to the graph.
        
        \item \textbf{Edges} are abstract representations of the connectivity between vertices. If there exists a path from one portal to another in the base occupancy grid without crossing any other portals, then we add an edge between the two portals. We initialize the cost with the minimum distance between the two vertices. 
\end{itemize}

\begin{figure}
    \centering
    \begin{subfigure}[t]{.49\linewidth}
    \centering
    \includegraphics[width=.8\linewidth]{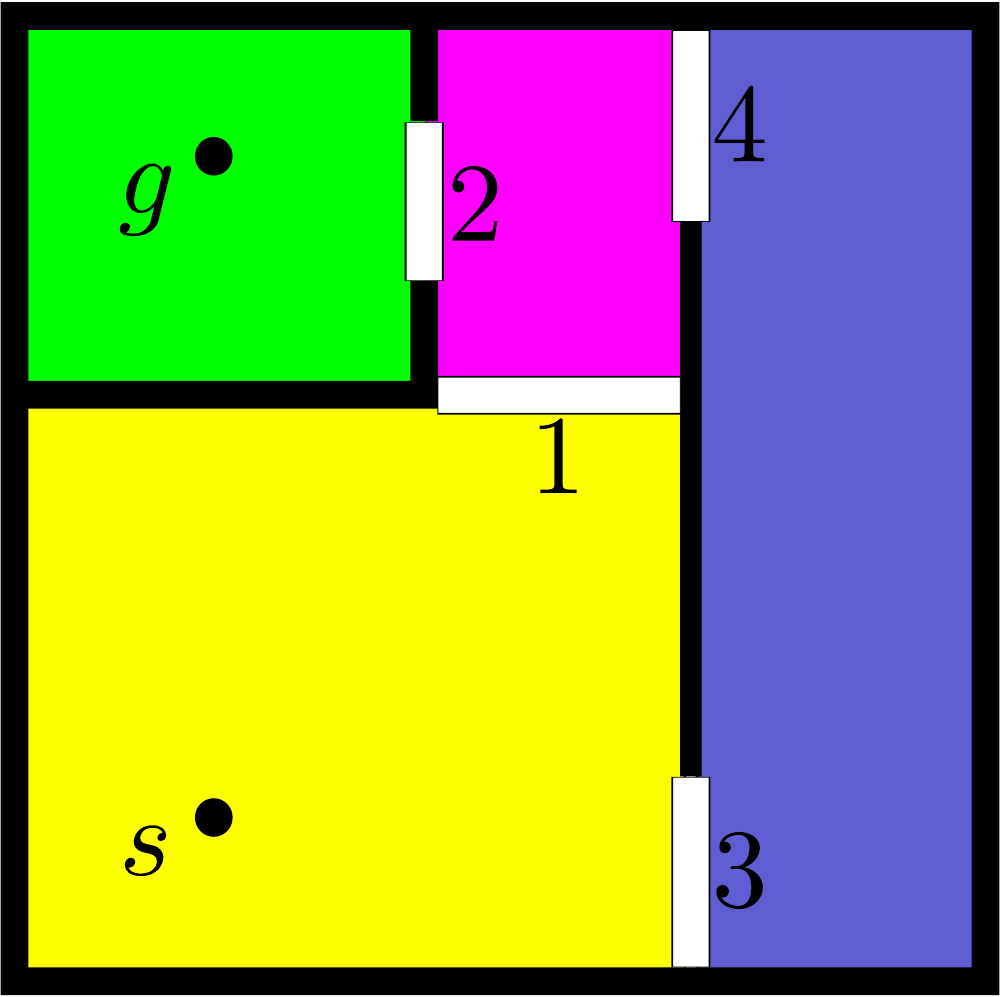}
    \end{subfigure}
    \begin{subfigure}[t]{.49\linewidth}
    \begin{tikzpicture}
\begin{scope}[every node/.style={circle,draw}]
	\node (1) [align=center] at (0,0) {1};
    \node (2) [align=center] at (-1,1) {2};
    \node (3) [align=center] at (1,-1) {3};
    \node (4) [align=center] at (1,1) {4};
    \node (5) [align=center] at (-1,-1) {s};
    \node (6) [align=center] at (-2,1) {g};
\end{scope}
\begin{scope}[every node/.style={circle},
              every edge/.style={draw=black}]
    \path [-,every node/.style={sloped,anchor=south,auto=false}] (1) edge[bend left=0] node[above] {} (2);
    \path [-,every node/.style={sloped,anchor=south,auto=false}] (1) edge[bend left=0] node[above] {} (3);    
    \path [-,every node/.style={sloped,anchor=south,auto=false}] (1) edge[bend left=0] node[above] {} (4);    
    \path [-,every node/.style={sloped,anchor=south,auto=false}] (1) edge[bend left=0] node[above] {} (5);
    \path [-,every node/.style={sloped,anchor=south,auto=false}] (2) edge[bend left=0] node[above] {} (6);
    \path [-,every node/.style={sloped,anchor=south,auto=false}] (2) edge[bend left=0] node[above] {} (4);
    \path [-,every node/.style={sloped,anchor=south,auto=false}] (3) edge[bend left=0] node[below] {} (1);
    \path [-,every node/.style={sloped,anchor=south,auto=false}] (3) edge[bend left=0] node[below] {} (4);
    \path [-,every node/.style={sloped,anchor=south,auto=false}] (3) edge[bend left=0] node[below] {} (5);
\end{scope}
\end{tikzpicture}
    \end{subfigure}
    \caption{Example of occupancy grid decomposition (left) and the resulting navigation graph (right).}
    \label{fig:navgraph}
\end{figure}
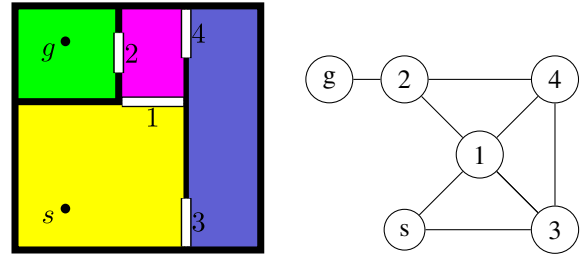

We use this navigation graph by instructing the robot to travel to the spatial coordinates of the portals corresponding to the vertices dictated by the policy. Localization is performed using existing algorithms such as AMCL \cite{pfaff2006robust} on the base occupancy grid. 
These algorithms are robust and can still localize reasonably well in the presence of unexpected obstacles. The robot will store the base occupancy grid, submap divisions, and the supermaps, which are subgraphs of the navigation graph. This allows the robot to remember changes in the environment that cause an edge to no longer be traversable, rather than storing a different occupancy grid for each new configuration.
This means that changes to the environment that do not significantly alter the path a robot would take to reach the goal are not stored in memory.

For ease of explanation, we represented a map $M$ as the tuple $(E^b,E^u)$. However, in practice $M$ is stored as an adjacency table representing a graph with edges marked as blocked, unblocked, or unknown. Therefore, we will also use $M$ to represent the navigation graph induced by $(E^b,E^u)$.

One concern is the possibility of creating a multigraph where multiple edges share the same endpoints, such as in Figure~\ref{fig:multigraph}. While this scenario does not need to be avoided, if an implementation allows for the construction of a navigation multigraph, then each edge will need an explicit label instead of representing it with the endpoints $(u,v) \ u,v \in V$, as is done traditionally. This enables the robot to differentiate the edges during planning and traversal.

\begin{figure}
    \centering
    \begin{subfigure}[t]{.49\linewidth}
    \centering
    \includegraphics[width=.7\linewidth]{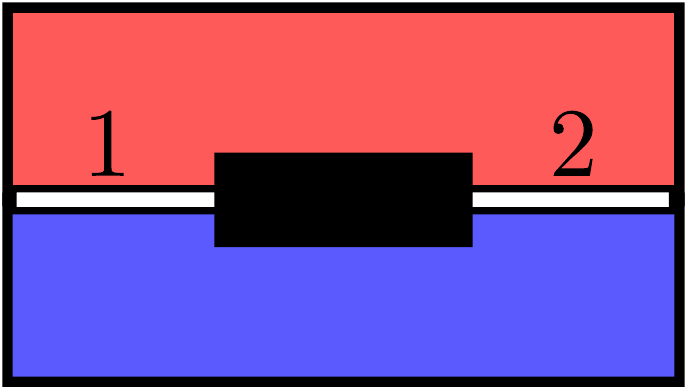}
    \caption{}
    \end{subfigure}
    \begin{subfigure}[t]{.49\linewidth}
    \centering
    \begin{tikzpicture}
\begin{scope}[every node/.style={circle,draw}]
	\node (1) [align=center] at (0,0) {1};
    \node (2) [align=center] at (2,0) {2};
\end{scope}
\begin{scope}[every node/.style={circle},
              every edge/.style={draw=black}]
    \path [-,every node/.style={sloped,anchor=south,auto=false}] (1) edge[bend left=25] node[above] {} (2);
    \path [-,every node/.style={sloped,anchor=south,auto=false}] (1) edge[bend right=25] node[above] {} (2); 
\end{scope}
\end{tikzpicture}
\caption{}
    \end{subfigure}
    \caption{Example of a decomposition that could form a multigraph.}
    \label{fig:multigraph}
\end{figure}
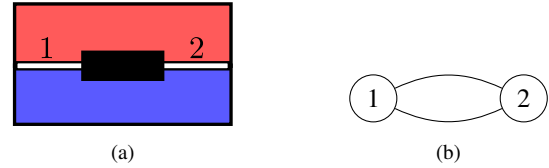

\subsection{Resolving Edges} \label{sec:resedges}
The robot requires a method to determine if an edge is blocked or unblocked, i.e., to \emph{resolve} the edge.
The up-to-date costmap and the current location of the robot can be utilized to update the edge states of the navigation graph during task execution. An edge can have one of three possible states: blocked, unblocked, or unknown. Resolving an edge means to set the state to be blocked or unblocked. The edges of the navigation graph are set to be unknown at the beginning of every task. While the robot is guaranteed to resolve an edge when it has traversed the entire length of the edge, a few properties of a 2D environment can be exploited to resolve an edge earlier.

\begin{figure}
    \centering
    \begin{subfigure}{.49\linewidth}
    \centering
    \includegraphics[width=.85\linewidth]{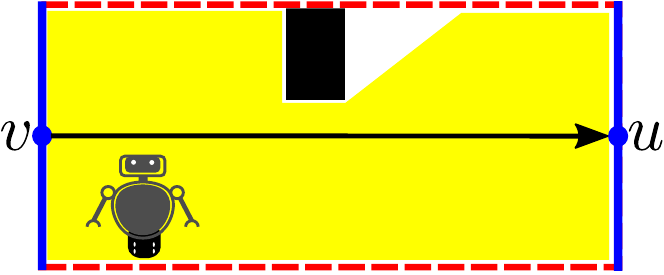}
    \caption{}
    \end{subfigure}
    \begin{subfigure}{.49\linewidth}
    \centering
    \includegraphics[width=.85\linewidth]{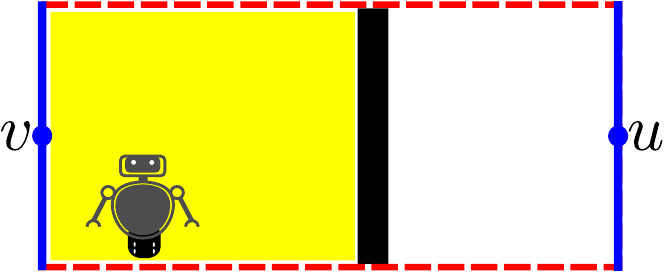}
    \caption{}
    \end{subfigure}
    \begin{subfigure}{.49\linewidth}
    \centering
    \includegraphics[width=.85\linewidth]{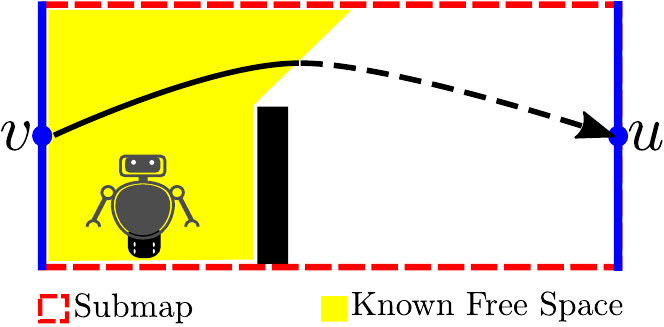}
    \caption{}
    \end{subfigure}
    \caption{Examples of different states for edge $(v,u)$ given what the robot knows: a) unblocked, b) blocked, c) unknown.}
    \label{fig:dir_inf_state}
\end{figure}

Figure~\ref{fig:dir_inf_state} shows a robot trying to traverse 3 different edges. The area that is observed by the robot is shown in yellow, and areas that the robot has not observed are shown in white, with obstacles that the robot observes and cannot traverse shown in black. In (a), the robot can see a path from $u$ to $v$, and so we can resolve the edge to be unblocked. In (b), the robot can see that there is no path from $u$ to $v$ in the submap, so we can resolve the edge as blocked. In (c), there may or may not be a path to goal in the unobserved part of the submap so the state of the edge is unknown.

Formally, let the \emph{true free space} be denoted by $C$ where $C \subseteq \mathbb{R}^2$. This is the area of the environment the robot can occupy without colliding with an obstacle. We assume a disc-shaped robot with a disc-shaped safety region so that orientation does not have to be considered. We also assume $C$ does not change throughout the task execution. Finally, $C$ is hidden from the robot.

Let the \emph{known free space} be denoted by $C_k$ where $C_k \subseteq C$. This is the area of the environment the robot has observed to be free space since the beginning of the task. Let the \emph{unknown space} be denoted by $C_u$ where $C_u \subseteq \mathbb{R}^2$. This is the area of the environment that the robot has not observed yet, but is marked as free space by the base occupancy grid. Note that $C_u \cap C_k = \emptyset$. Finally, let the \emph{optimistic free space} be denoted by $C_o$ where $C_o = C_k \cup C_u$. This is the area of the environment the robot believes it can occupy, assuming unknown space is also free. 
We will also assume that $C \subseteq C_o$, which means obstacles in the base occupancy grid will be assumed to always exist.

Let $S$ denote a submap, then $C(S), C_k(S), C_u(S), C_o(S)$ is the true free space, known free space, unknown space, and the optimistic free space respectively, in submap $S$. Let $\mathbf{x}_R \in \mathbb{R}^2$ be the current location of the robot.
Now a formal definition of an unblocked and blocked edge can be given: 

\begin{definition}[Unblocked and Blocked Edges]
 Given an edge $(u,v)\in E$ where $u,v \in V$, and $S(u,v) \subset \mathbb{R}^2$ is the submap associated with edge $(u,v)$, the edge is \emph{unblocked} if and only if there exists a path, $P \subseteq C(S(u, v))$ from $u$ to $v$. Otherwise the edge is \emph{blocked}.
\end{definition}

\begin{lemma} \label{lemma:bothvisible}
An edge $(u,v)$ is unblocked if there exists a path from $u$ to $v$ in $C_k(S(u, v))$. 
\end{lemma}
\begin{proof}
    Let $P(u,v)$ denote the path in $C_k(S(u, v))$. Since $C_k(S(u, v)) \subseteq C(S(u, v))$, then it follows that $P(u,v) \subseteq C(S(u, v))$.
\end{proof}

\begin{lemma} \label{lemma:pathblocked}
An edge $(u,v)$ is blocked if there does not exist a path from $u$ to $v$ in $C_o(S(u,v))$.
\end{lemma}
\begin{proof}
    By the definition of an unblocked edge, suppose there is a path $P(u,v) \subseteq C(S(u,v))$ but there does not exist a path from $u$ to $v$ in $C_o(S(u,v))$. This implies there exists $\mathbf{x} \in C(S(u,v))$ where $\mathbf{x} \notin C_o(S(u,v))$. But that is a contradiction since $C(S(u, v)) \subseteq C_o(S(u, v))$. Therefore, $P(u,v) \subseteq C_o(S(u,v))$.
\end{proof}

Algorithm~\ref{algo:dirInf} exploits these properties to resolve edge states. This algorithm is continuously run at a user-specified rate because $C_k(S)$, $C_u(S)$ and $C_o(S)$ are dependent on the time passed since the beginning of the task. In practice, the maximum rate of the algorithm depends on the computation time of the algorithm used to find a path in lines \ref{line:bothvis} and \ref{line:findpath}.

To avoid checking edges that are not immediately relevant to the robot, an edge has to meet at least one of the following conditions to be observed (line \ref{line:edgecheck}):
\begin{enumerate}
    \item \label{connAlgo:cond1} The edge is in the same submap that the robot is currently positioned in.
    \item \label{connAlgo:cond2} At least one of the endpoints of the edge are within the maximum observation range of the robot at $\mathbf{x_R}$.
\end{enumerate}

The maximum observation range of the robot is a circle with a radius of $r_{\max}$, which is based on sensor parameters and costmap implementation.

Condition~\ref{connAlgo:cond1} forces the algorithm to check the states of all the edges in the current submap. If this condition was not in place, the robot would not resolve edges in submaps that were larger than the maximum observation range. However, if the submap is smaller than the maximum observation range, then only checking edges in that submap would be short-sighted if the robot is capable of resolving edges in other submaps that are close by. Condition \ref{connAlgo:cond2} takes advantage of the robot's maximum observation range if possible.

\begin{algorithm}
  \KwIn{$C_k$, $C_u$, $r_{\max}$, $\mathbf{x_R}$}
    $C_o = C_k\cup C_u$\;
    $S_{\text{curr}} = $ \textsc{CurrentSubmap}($x_R$)\;
    \For{each $(u,v) \in E$}{
        \If{$||\mathbf{x_R}-u||_2 \leq r$ OR $||\mathbf{x}_R-v||_2 \leq r$ OR $S(u,v)=S_{\text{curr}}$}{ \label{line:edgecheck}
            \If{$u$ or $v \notin C_o(S(u,v))$}{ \label{line:inobst}
                Set $(u,v)$ as blocked\;
            }
            \ElseIf{\text{there exists} $P(u,v) \subseteq C_k(S(u,v))$}{ \label{line:bothvis}
                Set $(u,v)$ as unblocked\;
            }
            \ElseIf{\text{there does not exist} $P(u,v) \subseteq C_o(S(u,v))$}{ \label{line:findpath}
                Set $(u,v)$ as blocked\;
            }
        }
    }
 \caption{edgeResolver}
 \label{algo:dirInf}
\end{algorithm}

Line~\ref{line:inobst} checks if either of the endpoints of edge $(u,v)$ are in an obstacle; if they are in an obstacle, then the edge is marked as blocked. In chapter 5 of~\cite{Tsang2020}, an extension of edge resolution is presented to be more robust by considering the entire length of a portal instead of just the center point.

Line~\ref{line:bothvis} exploits Lemma~\ref{lemma:bothvisible} to set $(u,v)$ as unblocked.  In line~\ref{line:findpath}, a path planner is run on $C_o(S(u,v))$ to find a path from $u$ to $v$. If no path is found, then the edge is blocked. When selecting a path planner, it only needs to return if a path exists or not, thus optimality is not a concern. Minimizing the computation speed should be the primary objective. Planners that use heuristics (such as A*) to speed up computation are good options.

If a path is found, but $u$ and $v$ do not satisfy the conditions in lines~\ref{line:inobst} or~\ref{line:bothvis}, this implies the robot still does not know whether the edge is blocked or unblocked, since the path could be blocked by an unseen obstacle, as illustrated by (c) in Figure~\ref{fig:dir_inf_state}. Then edge $(u,v)$ cannot be resolved and the algorithm moves on to checking the next edge.

Edges resolved during task $t$ are used to update the map $M_t=(E^b_t,E^u_t)$. Computationally fast path planning is crucial to running edge resolver at the desired rate. Since the robot's current location is an input to Algorithm~\ref{algo:dirInf}, if the path finder (lines~\ref{line:bothvis} and~\ref{line:findpath}) is slow, it can result in edges not being resolved because the robot has moved far away by the time Algorithm~\ref{algo:dirInf} is re-run. Large submaps may also slow down the the rate of the edge resolver as most path finders will take longer to run on large maps.

\begin{figure*}
    \begin{subfigure}[t]{.16\linewidth}
    \begin{center}
        \includegraphics[width=.9\linewidth]{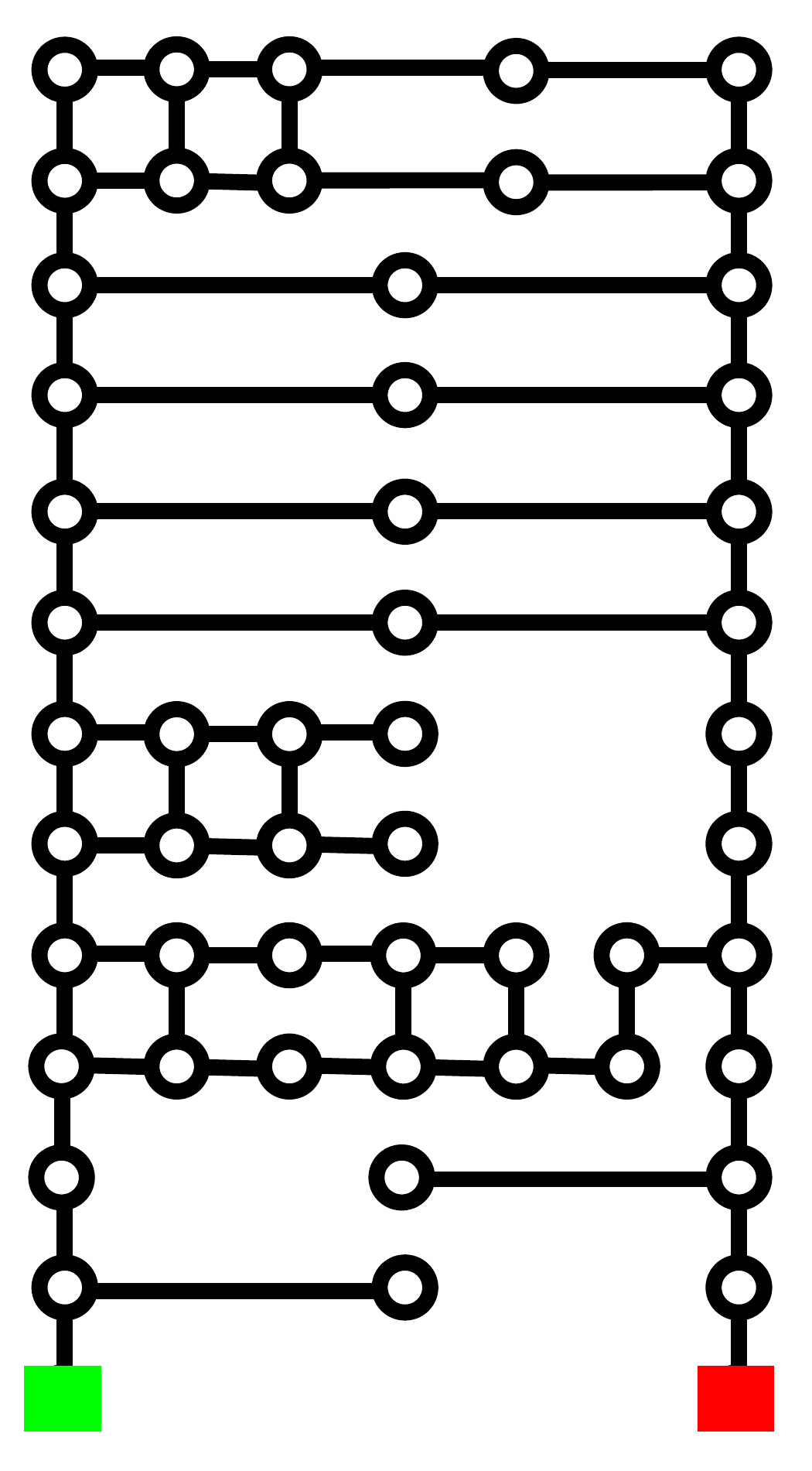}
    \end{center}
    \caption{}
    \label{subfig:grocery1}
    \end{subfigure}
    \begin{subfigure}[t]{.16\linewidth}
    \begin{center}
        \includegraphics[width=.9\linewidth]{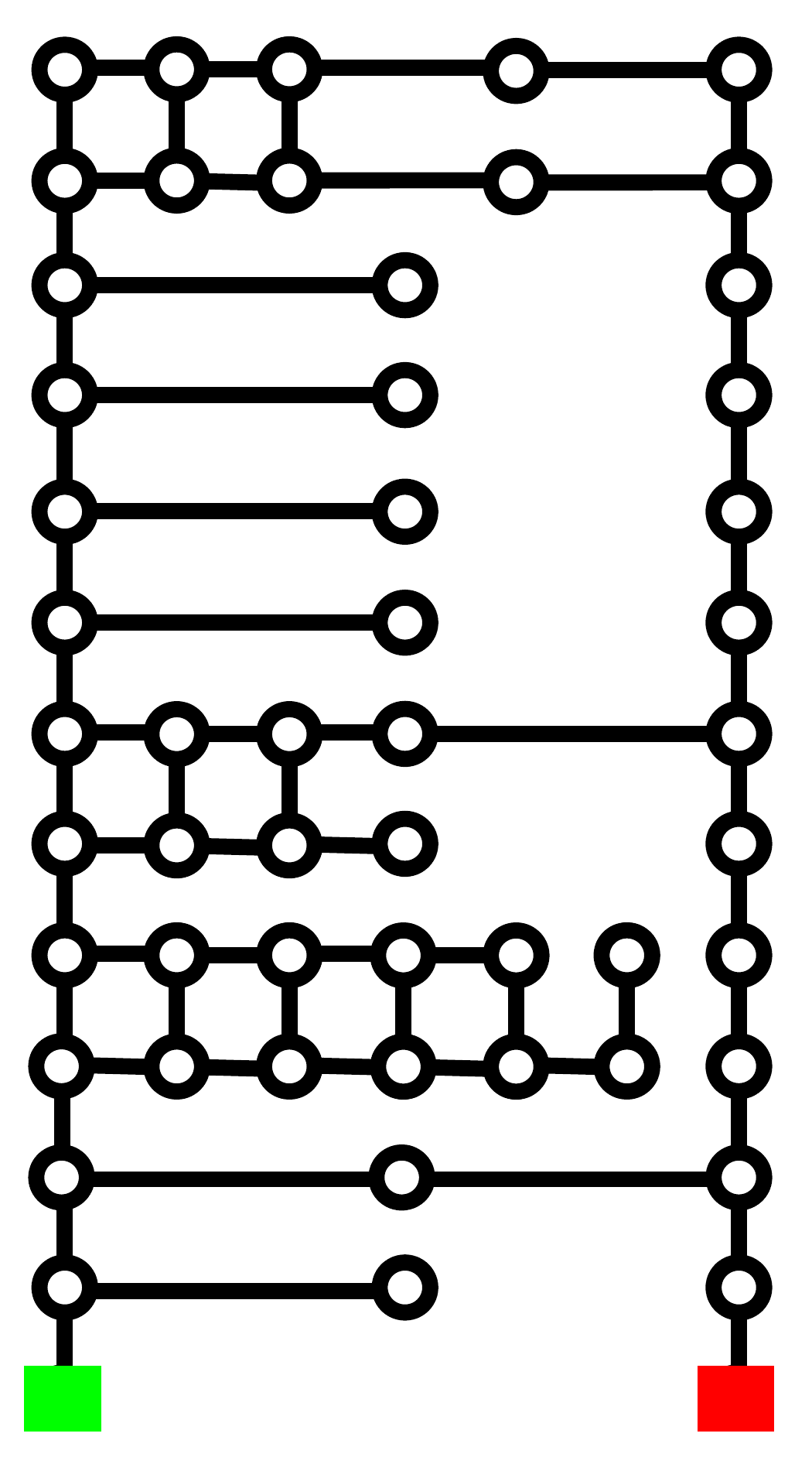}
    \end{center}
    \caption{}
    \label{subfig:grocery2}
    \end{subfigure}
    \begin{subfigure}[t]{.16\linewidth}
    \begin{center}
        \includegraphics[width=.9\linewidth]{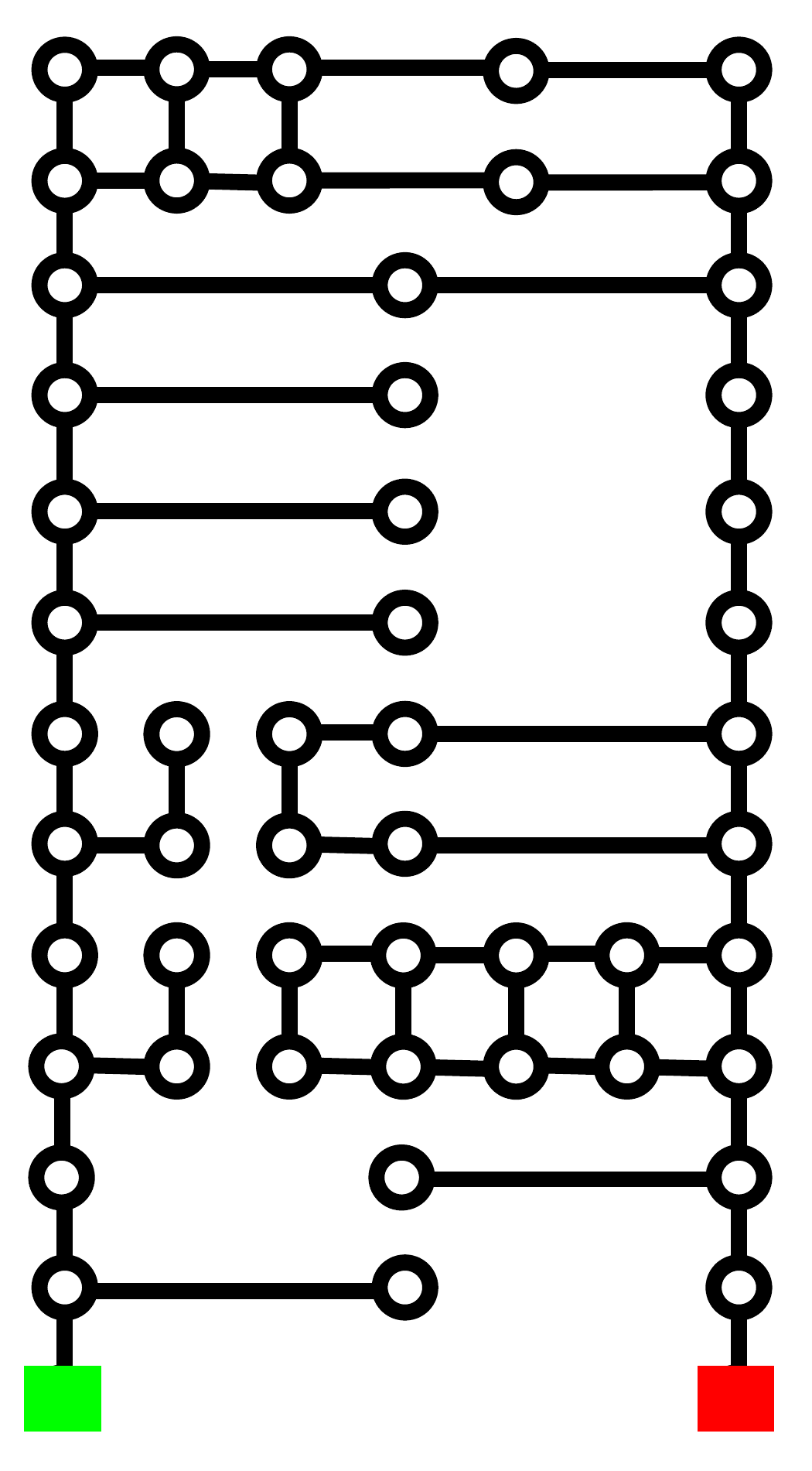}
    \end{center}
    \caption{}
    \label{subfig:grocery3}
    \end{subfigure}
    \begin{subfigure}[t]{.16\linewidth}
    \begin{center}
        \includegraphics[width=.9\linewidth]{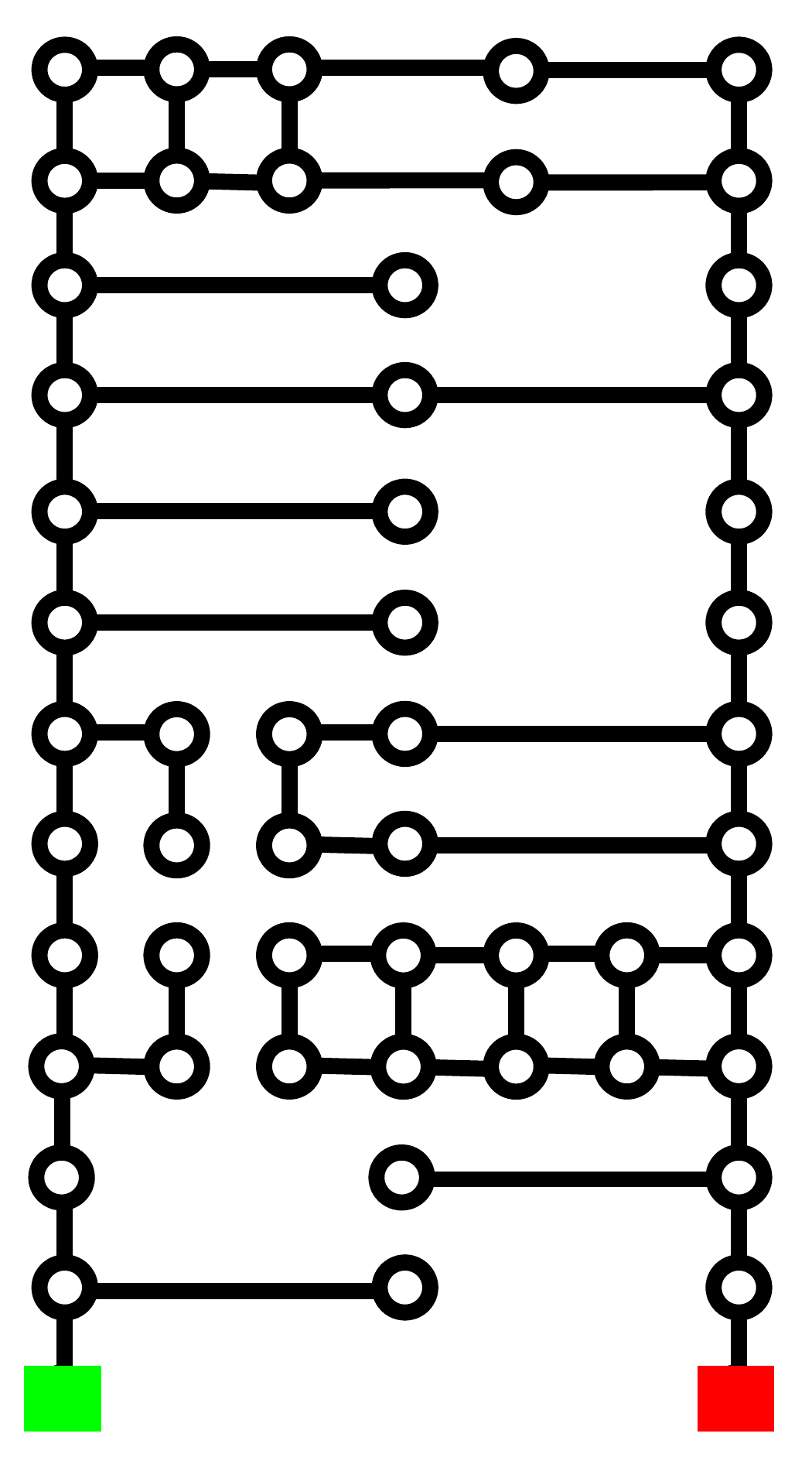}
    \end{center}
    \caption{}
    \label{subfig:grocery4}
    \end{subfigure}
    \begin{subfigure}[t]{.16\linewidth}
    \begin{center}
        \includegraphics[width=.9\linewidth]{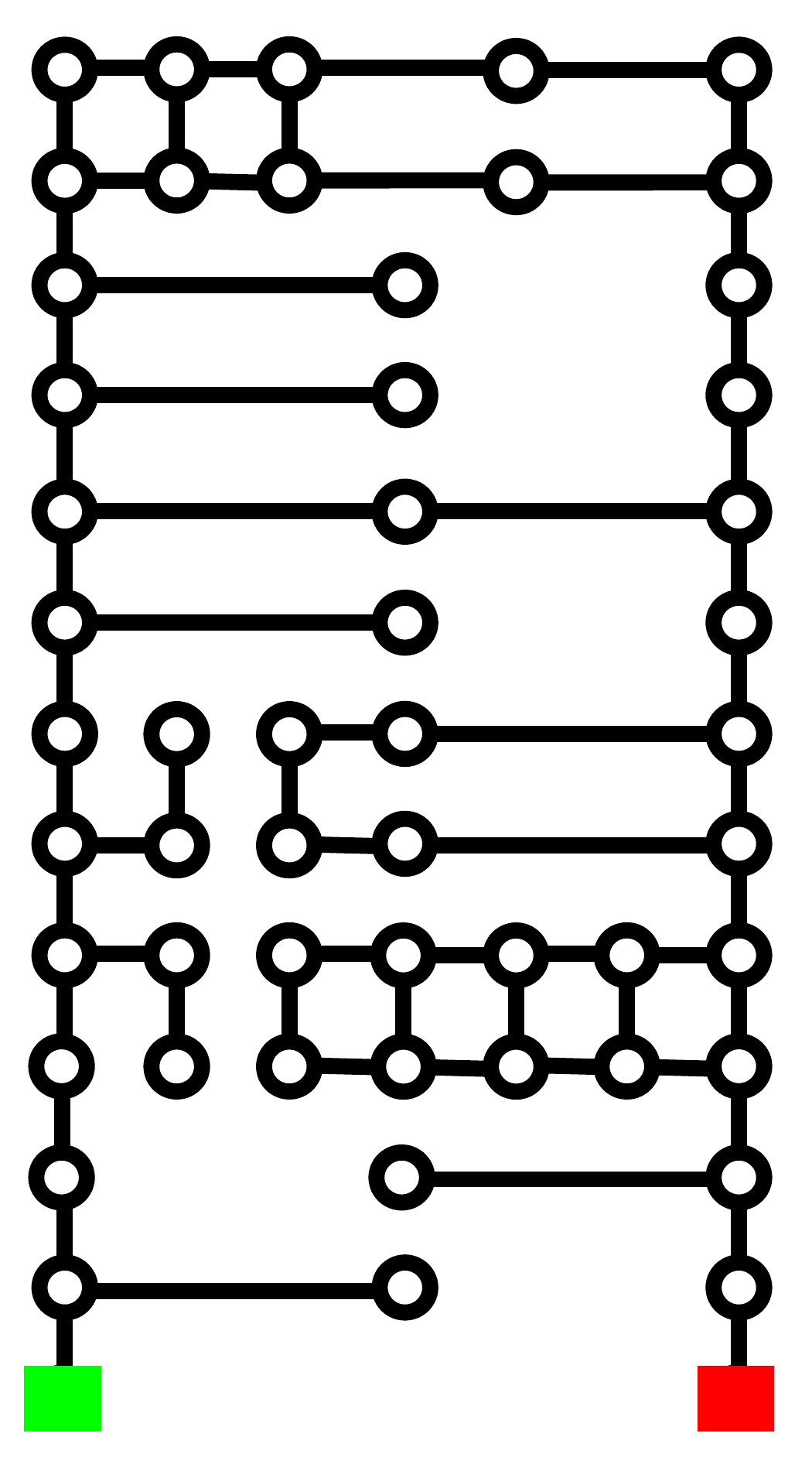}
    \end{center}
    \caption{}
    \label{subfig:grocery5}
    \end{subfigure}
    \begin{subfigure}[t]{.16\linewidth}
    \begin{center}
        \includegraphics[width=.9\linewidth]{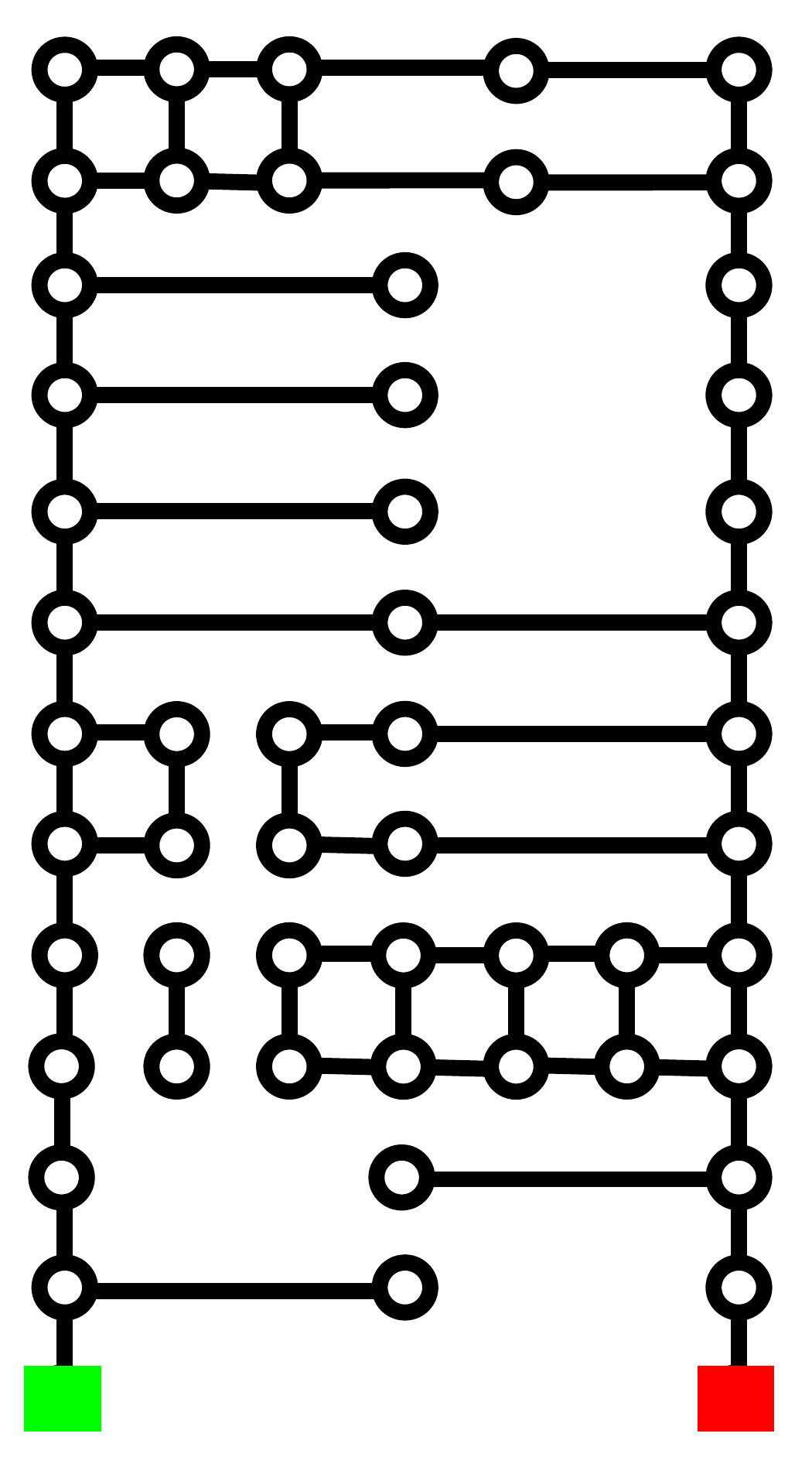}
    \end{center}
    \caption{}
    \label{subfig:grocery6}
    \end{subfigure}
    \caption{Realizations of test environment 1 that the robot may operate in to traverse from start to goal, shown in green and red respectively}
    \label{fig:simenvironmentc}
\end{figure*}

\subsection{Policy Execution}
The RPP-Hybrid policy returned in Section~\ref{subsec:LRPPsol} assumes that the robot can resolve an edge at an edge endpoint. As the previous subsection has shown, this may not always be true. As a result, minor modifications need to be made to the policy which can be applied during the policy execution.

In the RPP-Hybrid policy, when observing an edge (ex. $(v,u)$), the leg ends at one of the edge's endpoints (ex. $v$). To ensure that the robot completes the observation (i.e. resolves the edge), $u$ needs to be appended to the leg. The robot may or may not arrive at $u$ before completing the observation. After the observation is completed, the robot can return to $v$ or if the next leg in the policy begins with the edge $(v,w)$, it can set $w$ as the next vertex to move to.

When moving from one vertex to the next, the robot may not be at a vertex when it discovers an edge is blocked, so care must be taken when recalculating a path for the optimistic policy. Suppose the robot is traversing edge $(v,u)$, and discovers it is blocked. We first insert a temporary vertex $v_R$ containing the robot's current position $\mathbf{x}_R$ into the navigation graph. Second, a temporary edge is inserted for each edge from $v$ to another portal in the same submap that is not blocked. An edge is also added to $v$. The cost of each edge is the euclidean distance to $\mathbf{x}_R$, i.e., $||\mathbf{x}_R -u||_2$. Finally, we run the shortest path algorithm in the modified graph with $v_R$ as the starting vertex and $v_g$ as the goal.

\section{Results} 

In this section, we evaluate the performance of the RPP-Hybrid policy and the LAMP framework in the following experiments:  1) the RPP-Hybrid policy is simulated in a discrete environment and compared with the existing UCT-CTP algorithm~\cite{Eyerich2010} and the factor graph approach described in~\cite{nardi2020long}; 2) LAMP is simulated in a continuous environment with a real navigation stack to verify that our assumptions hold; and 3) LAMP is implemented in a scenario with a real robot, demonstrating the performance in real-world conditions such as non-ideal localization. The implementation of the RPP-Hybrid policy used in the experiments uses the simple approach to map merging where new maps are merged with the first agreeing super map. We have included a supplementary video showing trials from the LAMP simulations and the real robot experiment. This will be available at http://ieeexplore.ieee.org.

\label{sec:implresults}

\subsection{RPP-Hybrid Simulation Results} \label{sec:lrppresults}

\begin{figure}
    \begin{subfigure}[t]{.49\linewidth}
    \begin{center}
        \includegraphics[width=.5\linewidth]{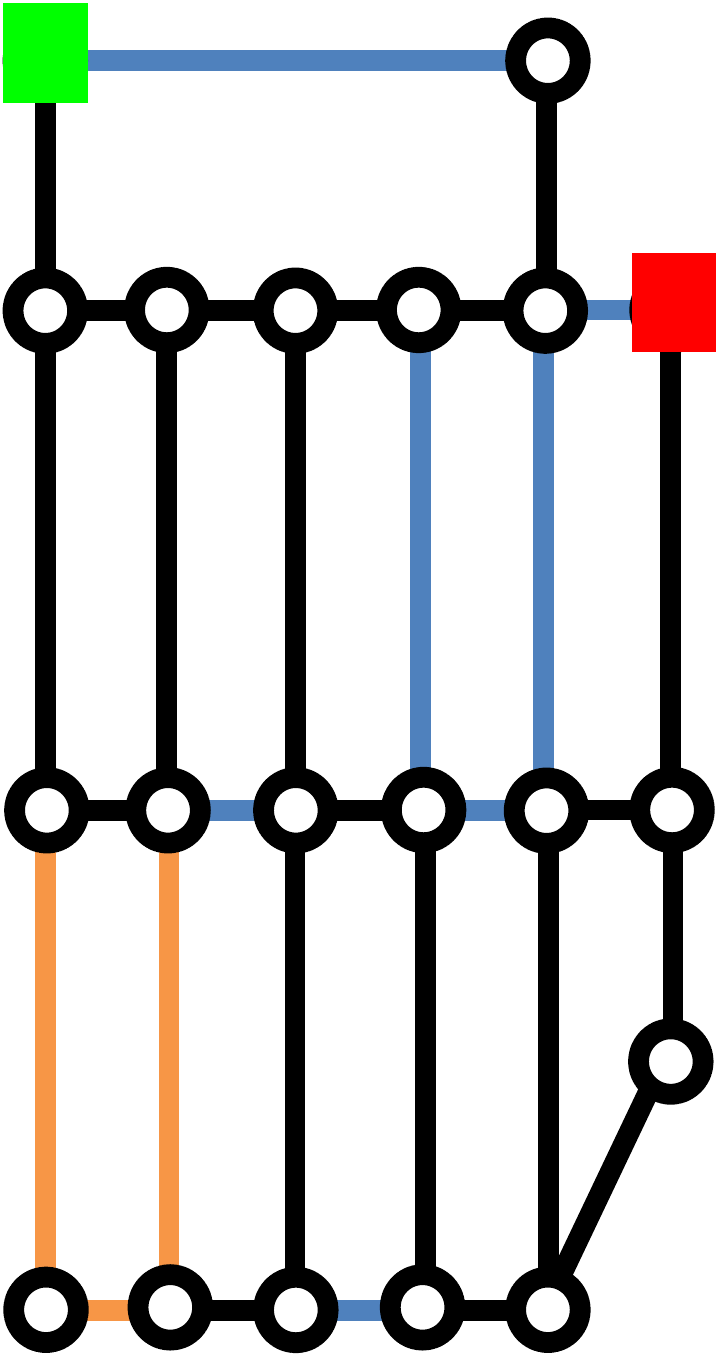}
    \end{center}
    \caption{Test environment 2}
    \label{subfig:beauty}
    \end{subfigure}
    \begin{subfigure}[t]{.49\linewidth}
    \begin{center}
        \includegraphics[width=.5\linewidth]{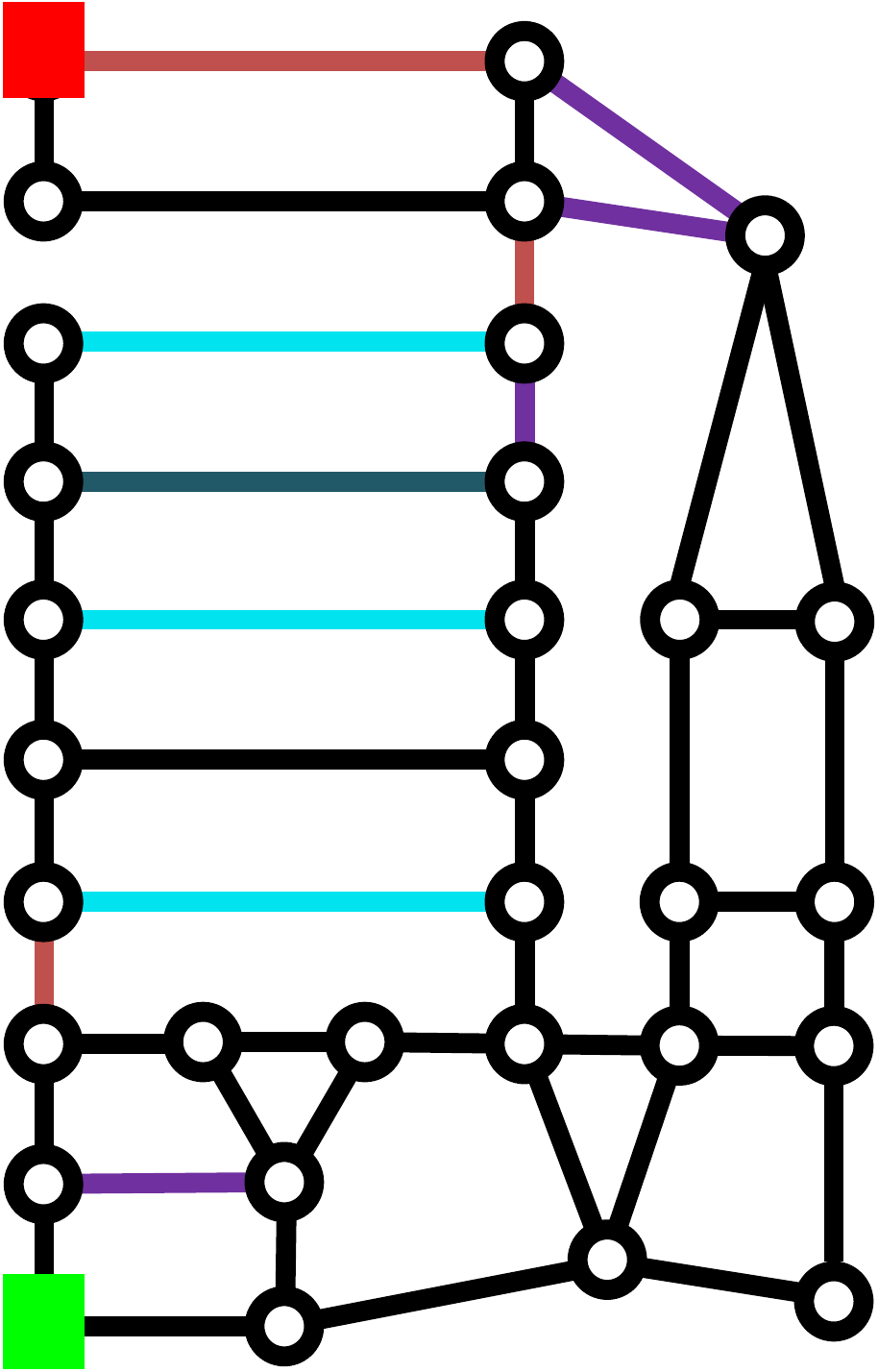}
    \end{center}
    \caption{Test environment 3}
    \label{subfig:produce}
    \end{subfigure}
    \caption{Additional test environments. Coloured edges are blocked in different realizations according to  Table~\ref{tab:simconfigs}.}
    \label{fig:simenvironments}
\end{figure}

To evaluate the RPP-Hybrid policy, simulations were run and the performance of RPP-Hybrid was compared to a purely optimistic planner, an implementation of the UCT-CTP algorithm using independent edge probabilities, and the factor graph approach described in~\cite{nardi2020long}. The optimistic planner is an implementation of A* that initially assumes all edges are unblocked and updates the map and replans when it encounters a blocked edge, with no information saved between each trial. The implementation of UCT-CTP keeps count of the number of times it has observed each edge in either state and calculates independent probabilities for each edge. The algorithm utilizes these probabilities to compute a new policy before each task. The factor graph approach maintains a compact approximation of the joint probability mass function of the edges, and uses rollouts similar to UCT-CTP to generate a path. In these experiments, we only use the observations along the robot's path to update the robot's model of the environment between tasks. The UCT-CTP approach, factor graph approach, and the optimistic approach are all calculated online. The RPP-Hybrid policy is computed offline (i.e., prior to each task execution), however an online version of the policy is described in \cite{Tsang2020}.

The environments for these simulations are graphs that were manually drawn from the floor-plan of a grocery store and are shown in Figures~\ref{fig:simenvironmentc} and~\ref{fig:simenvironments}. Different realizations of Environment 1 are given in Figure~\ref{fig:simenvironmentc} and different realizations of Environments 2 and 3 are given in Table~\ref{tab:simconfigs}. Each algorithm was run in $100$ trials of $100$ tasks each and was given the same sequence of realizations, which was obtained by randomly selecting a realization for each trial with uniform probability.

 \begin{table}
 \caption{Different realizations of test Environments 2 and 3.}
 \label{tab:simconfigs}
 \centering
 \begin{tabular}{c c l}
     \toprule
     Environment & Realization & Blocked Edges  \\
     \midrule
     \multirow{3}{*}{2} & 1 & none \\ & 2 & Blue \\ & 3 & Blue, Orange \\
     \midrule
     \multirow{4}{*}{3} & 1 & none \\ & 2 & Orange \\ & 3 & Blue, Cyan \\ & 4 & Blue, Purple \\
     \bottomrule
 \end{tabular}
 \end{table}

 \begin{table}
 \caption{Policy Cost as a Percentage of Optimal.}
 \label{tab:simcostcomparison}
 \centering
 \begin{tabular}{c l c c}
     \toprule
    Environment & Policy & Average (\%) & Last 10 Trials (\%) \\
    \midrule
    \multirow{4}{*}{1} & Optimistic & $218$ & $219$ \\
     & UCT-CTP & $218$ & $218$ \\
     & Factor Graph & $207$ & $179$ \\
     & RPP-Hybrid & $129$ & $117$ \\
     
    \midrule
    \multirow{4}{*}{2} & Optimistic & $110$ & $110$ \\
     & UCT-CTP & $114$ & $113$ \\
     & Factor Graph & $117$ & $110$ \\
     & RPP-Hybrid & $100$ & $100$ \\
     \midrule
    \multirow{4}{*}{3} & Optimistic & $123$ & $122$ \\
     & UCT-CTP & $129$ & $125$ \\
     & Factor Graph & $131$ & $121$ \\
     & RPP-Hybrid & $107$ & $104$ \\
     \bottomrule
 \end{tabular}
 \end{table}

\begin{figure}
    \begin{center}
    \includegraphics[width=\linewidth]{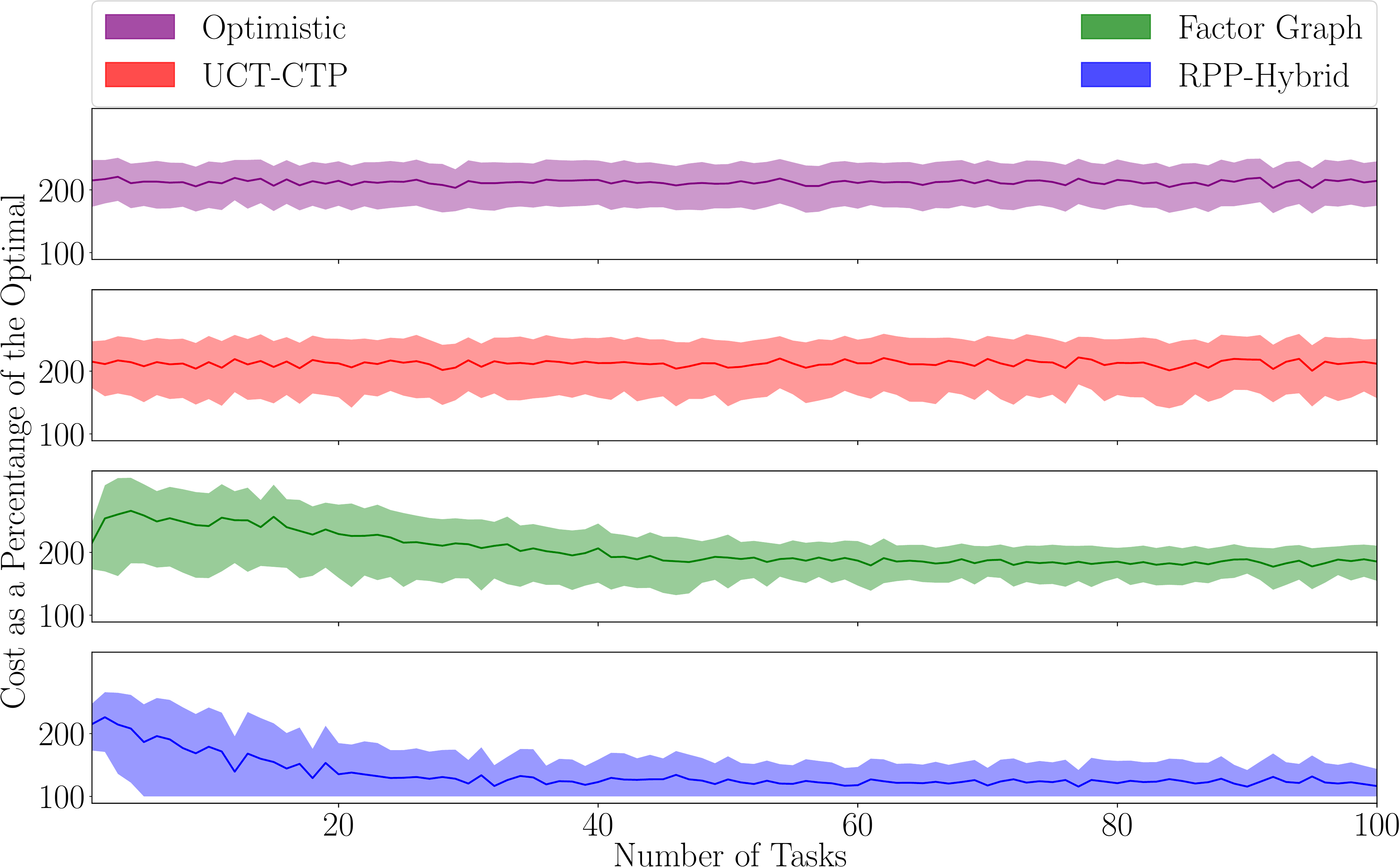}
    \end{center}
    \caption{Cost of optimistic, UCT-CTP, factor graph, and RPP-Hybrid policy relative to the optimal cost in Environment 1.}
    \label{fig:policycomparison}
\end{figure}

In these experiments, we consider the optimal cost to be the cost of the shortest path in the given realization. The average costs compared to the optimal are shown in Table~\ref{tab:simcostcomparison}, as well as the average cost of the last $10$ trials to show the performance each algorithm converged to. Figure~\ref{fig:policycomparison} shows the average cost of each policy in Environment 1 as the number of tasks completed increases.

The UCT-CTP policy with independent edge probabilities performs no better than the optimistic policy in all environments, and is not able to make significant improvements to its policy as it completes tasks. The Factor Graph approach, although approximating the correlation between the obstacles, is initially outperformed by the simple policies like the optimistic policy and UCT-CTP which ignore the correlations in the environment. This method eventually provides better solutions as it completes more tasks and builds a better model of the environment. Similar behaviour of the Factor Graph approach is shown for different environments in~\cite{nardi2020long}. The RPP-Hybrid approach outperforms all the policies in every environment. In Environment 1, it performs $89\%$ better than the optimistic and UCT-CTP approaches, and $78\%$ better than the Factor Graph approach. Also, its performance does not suffer in the first few trials as it learns the environment. Observe that the solution provided by the RPP-Hybrid policy approaches the optimal cost in Environments 2 and 3 and near optimal cost in Environment 1 with a sufficient number of task executions. 

One shortcoming of the UCT-CTP algorithm, factor graph approach, and the RPP-Hybrid policy is that with incomplete observations, they can converge to a sub-optimal policy where the cost of taking informative observations that would improve the policy is too high according to the model of the environment that the robot has. The UCT-CTP algorithm and the factor graph approach, both of which use rollouts, are particularly susceptible to this. Another weakness of the factor graph approach in this scenario is that it is loosely structured and lacks mechanisms to prevent excessive exploration during the earlier trials, and to purposely exploit known correlations in later trials. The RPP-Hybrid approach switches to an optimistic policy when it encounters a previously unseen environment, allowing its worst-case performance to stay close to that of the optimistic policy.

The runtime of each algorithm is primarily dictated by its number of shortest-paths calculations. The optimistic policy only performs this calculation when it observes a blocked edge on its current path, and was able to complete the 10000 tasks in Environment 1 in under a minute. The RPP-Hybrid policy performs this calculation once for every super map, and thus the runtime scales with the number of supermaps collected. In our experiments, since the number of supermaps remains low, the runtime is in the same order of magnitude of the optimistic policy, completing the same tasks in approximately 8 minutes. The UCT-CTP algorithm and factor graph approach compute a shortest path for each rollout, which presents a tradeoff: with the number of rollouts set in the 50-100 range, the algorithms can be implemented online.  However, the authors of the original UCT-CTP \cite{Eyerich2010} algorithm note that performance continues to improve past 1000 rollouts, beyond what is feasible on a modern desktop CPU. With the number of rollouts set to 100 in our experiments, the two algorithms were each able to complete the 10000 tasks for Environment 1 in approximately 10 hours.

\subsection{LAMP Simulation Results} \label{sec:lampsimresults}

 \begin{table}
 \caption{Existing algorithms used in experiment.}
 \label{tab:ros}
 \centering
 \begin{tabular}{l l p{35mm}}
     \toprule
     Component & ROS Package & Algorithm \\
     \midrule
     Global Planner & global\_planner & A* \\
     Local Planner & base\_local\_planner & Dynamic Window Approach \cite{Dieter1997} \\
     Localization & amcl & Augmented Monte Carlo Localization \cite{pfaff2006robust} \\
     Costmap & costmap\_2d & Layered costmap \cite{Lu2014}\\
     \bottomrule
 \end{tabular}
 \end{table}

\begin{figure}
    \centering
    \begin{subfigure}{.49\linewidth}
        \centering
        \includegraphics[width=.9\linewidth]{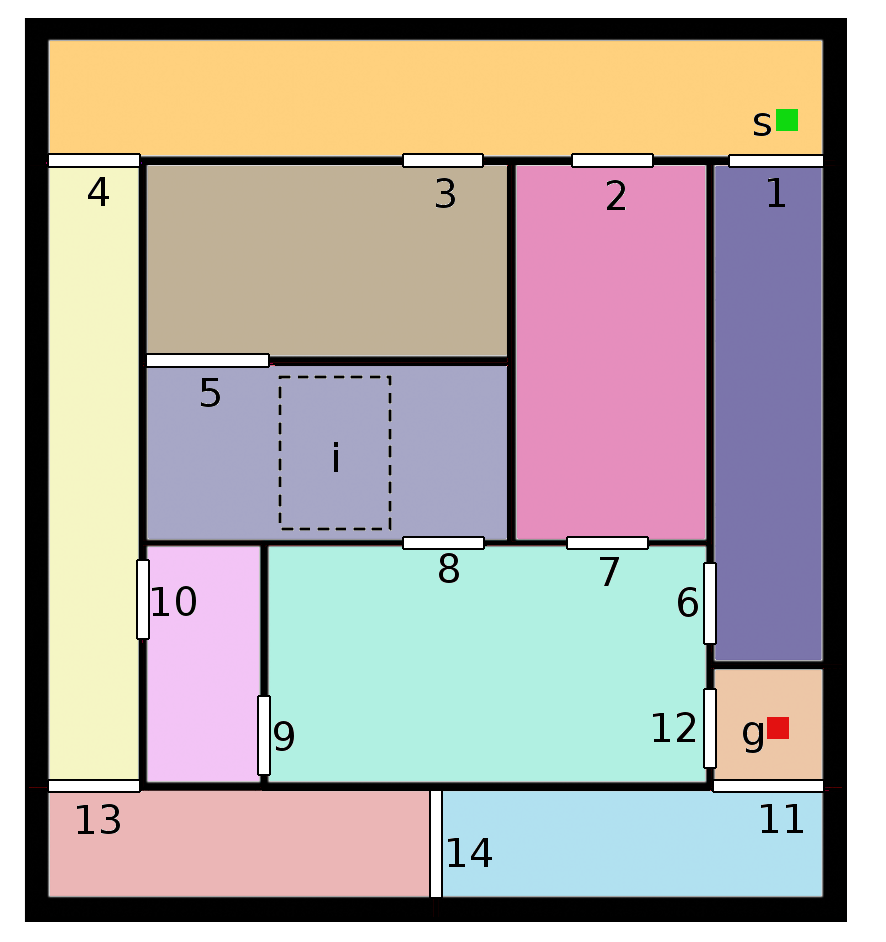}
        \caption{}
        \label{subfig:simmap}
    \end{subfigure}
    \begin{subfigure}{.49\linewidth}
        \centering
        \includegraphics[width=.9\linewidth]{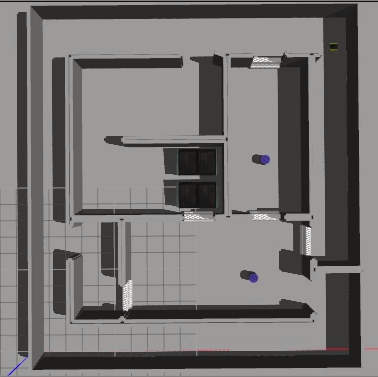}
        \caption{}
        \label{subfig:simex}
    \end{subfigure}
    \caption{(a) Base occupancy grid of the Gazebo test environment ($20$ m $\times$ $20$ m) showing submap decomposition, labelled portals, and a potential obstacle $i$. The start and goal vertices are labelled as $s$ and $g$ respectively. (b) Example of environment configuration with obstacles. In this example, obstacles appeared in $2, 6, 7, 8, 9$, and $i$, along with some random debris.}
    \label{fig:simex}
\end{figure}

An office environment was constructed and simulated in Gazebo. The navigation graph used for policy generation was manually constructed such that it is not a multigraph. All of the code for the policy generator and edge observer is in Python 2.7, which is integrated as a package for ROS Kinetic. The Visilibity library \cite{VisiLibity1:2008} was used to estimate the known free space for the edge observer. The LAMP framework was built on top of the ROS navigation stack, which already includes all the components in the standard navigation stack in Figure~\ref{fig:overall}. In the ROS navigation stack, the path planner has two components for planning and executing trajectories, namely global planner and local planner. Table~\ref{tab:ros} outlines which ROS packages and therefore which algorithms were used for each component of the navigation stack in our implementation. Our high level planner uses the move\_base action library to interact with the ROS navigation stack. The simulations were run on a PC with a 4.2 GHz Intel Core i7 processor with 32 GB memory and a GeForce GTX 1060 GPU with 6 GB of VRAM, using the ROS Gazebo simulator. The robot platform is a Clearpath Jackal (see Figure~\ref{fig:front_image}), a small differential drive unmanned ground vehicle (0.5m $\times$ 0.5m footprint) that has been equipped with a SICK LMS111 LiDAR. Figure~\ref{fig:simex} depicts a top-down view of the environment the robot operated in and Table~\ref{tab:sim_prob} shows the probability of edge blocking obstacles appearing. Each column of obstacles in Table~\ref{tab:sim_prob} is perfectly correlated, meaning either all the obstacles in the column are present, or none of them are present. If an obstacle appears on a portal, a white barrier blocks the portal, as shown in Figure~\ref{subfig:simex}. If the region $i$ is to have an obstacle appear, then two rectangular obstacles can appear anywhere in the submap, but in the same orientation as in Figure~\ref{subfig:simex}. In addition, to simulate scattered debris, up to 10 blue cylinders will spawn in random locations throughout the environment that will not affect the traversability of submaps, but will require the local planner to react.

The occupancy grid was converted into a graph with 16 vertices and 33 edges. The uncertainties in Table~\ref{tab:sim_prob} resulted in 32 possible configurations for the test environment, not accounting for all the different configurations the blue cylinders could be in.

The robot is tasked with repeatedly going from $s$ to $g$. We ran $10$ trials with $100$ task executions each. A selection of these trials are shown in the supplementary video. The environment configuration for each execution was randomly selected according to the probabilities in Table~\ref{tab:sim_prob}. We tested three policies for each trial: Simple, Optimistic, and RPP-Hybrid. The Simple policy does not use the navigation graph, we simply send the goal position $g$ to the navigation stack and the path planner guides the robot to the goal. This is the baseline we compare our policies to. The reasoning is that this is available out-of-the-box by ROS and presumably many robots will use this method of navigation.

Figure~\ref{fig:lampsavings} shows the average distance travelled by the RPP-Hybrid policy after performing a certain number of task executions, as a percentage of the distance travelled by the simple policy. The results show that the RPP-Hybrid outperforms the optimistic and the simple policy by $20-30\%$ and the mean performance improves monotonically over time. Note that the high-level policies for the simple policy and the optimistic policy are similar, however, the slight difference in their performance is due to small differences in distances in the costmap used by the simple policy and the navigation graph used by the optimistic policy.

Figure~\ref{fig:nsupermaps_optcalls} shows the rate of switches to the Optimistic policy decreasing exponentially as the number of completed task executions increases, which is the desired behavior.

\begin{table}
    \caption{Probabilities of obstacles appearing in the simulation.}
    \label{tab:sim_prob}
    \centering
    \begin{tabular}{lccccc}
         \toprule
         Obstacles &1 &7, 8, 9 &2, 6 & 3, 10, 12& $i$ \\
         \midrule
         Probability & 0.1&0.4& 0.5&0.6&0.9\\
         \bottomrule
    \end{tabular}
\end{table}

\begin{figure}
    \centering
    \includegraphics[width=.95\linewidth]{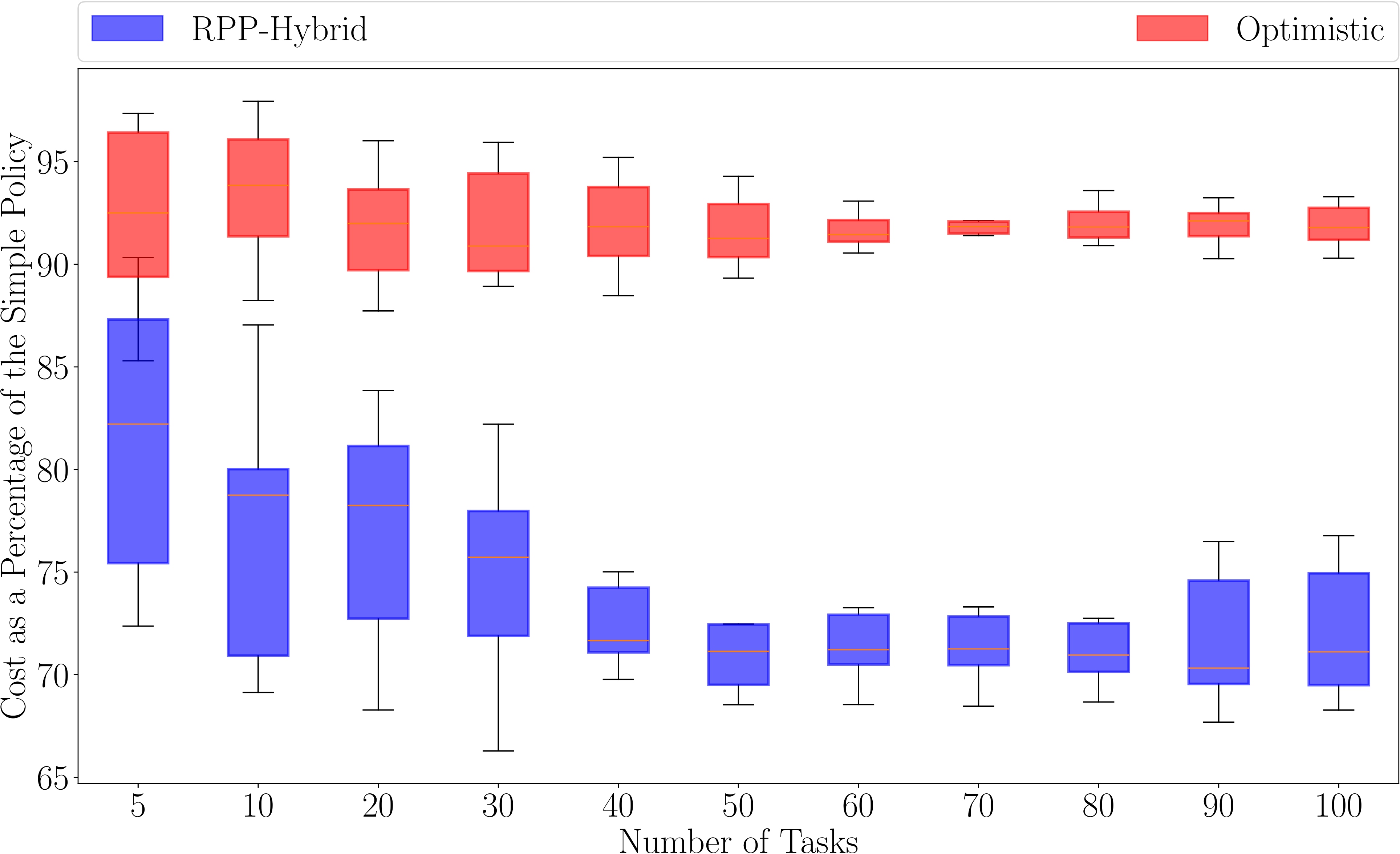}
    \caption{Cost savings of RPP-Hybrid policy normalized to the simple policy.}
    \label{fig:lampsavings}
\end{figure}

The map filter reduces the number of environment configurations that must be remembered by only storing maps with important differences. Figure~\ref{fig:nsupermaps_optcalls} shows the number of super maps plateauing as the number of completed task executions increases. In all the trials, the maximum number of super maps that were stored after 100 tasks was 20. Figure~\ref{fig:simagr} contains two examples of different environments where the maps generated by the robot were in agreement, resulting in considerably fewer maps being stored than if we were to compare the occupancy grids directly.

\begin{figure}
    \centering
    \includegraphics[width=.95\linewidth]{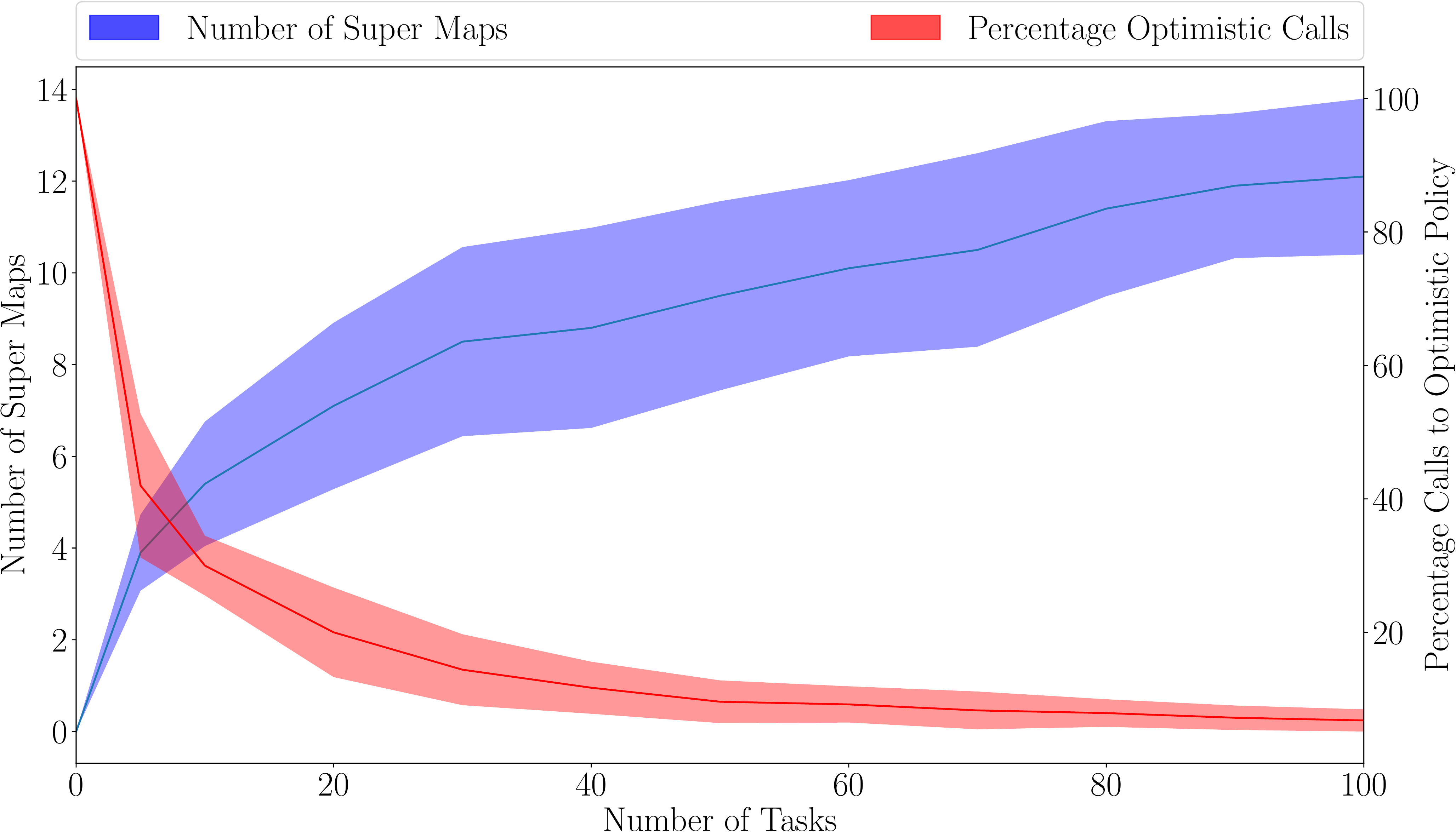}
    \caption[Plots of number of super maps and ratio of switches to the optimistic policy in LAMP simulation]{(blue) Average number of super maps stored compared to the number of tasks that have been executed. (red) Percentage of executed tasks where the high level planner switched to the optimistic policy.}
    \label{fig:nsupermaps_optcalls}
\end{figure}

\begin{figure}
    \centering
    \begin{subfigure}{.49\linewidth}
        \centering
        \includegraphics[width=.9\linewidth]{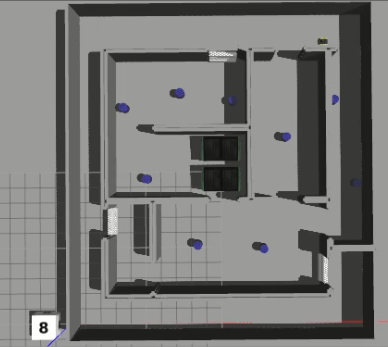}
        \caption{}
        \label{subfig:simagr_a1}
    \end{subfigure}
    \begin{subfigure}{.49\linewidth}
        \centering
        \includegraphics[width=.9\linewidth]{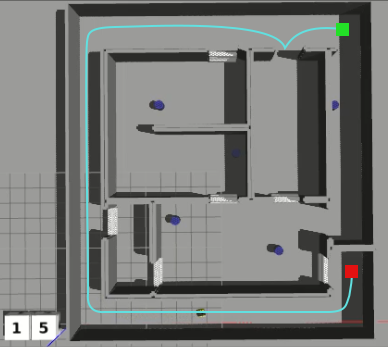}
        \caption{}
        \label{subfig:simagr_a2}
    \end{subfigure}
    \begin{subfigure}{.49\linewidth}
        \centering
        \includegraphics[width=.9\linewidth]{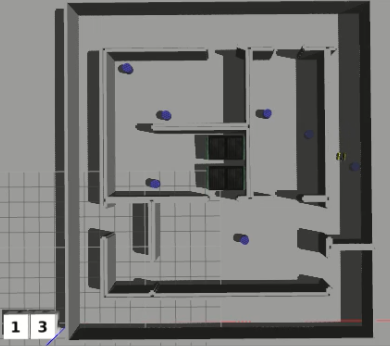}
        \caption{}
        \label{subfig:simagr_b1}
    \end{subfigure}
    \begin{subfigure}{.49\linewidth}
        \centering
        \includegraphics[width=.9\linewidth]{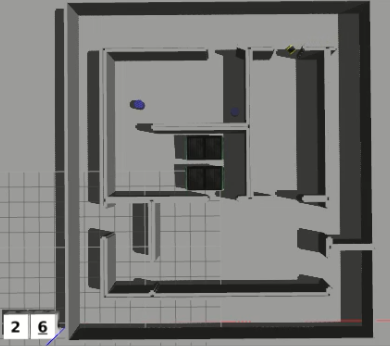}
        \caption{}
        \label{subfig:simagr_b2}
    \end{subfigure}
    \caption[Examples of map agreement in LAMP simulation]{Two examples of different environment configurations that satisfy map agreement. The top environments agree because in task 15, the robot did not see the barriers in the center along its route (light blue). The bottom environments agree despite the blue cylinders being in different locations.}
    \label{fig:simagr}
\end{figure}

Although minimizing computation time is not the focus of this work, it should be noted that during the simulations, the edge observer was operating at a rate of roughly 3 Hz, or observing 3 edges per second. While this was slower than desired, resulting in some edges remaining unknown when they should have been resolved, the RPP-Hybrid policy still performed better than the Simple policy. However, there is room for optimization and employing a faster path finder than A* would likely increase the edge observation rate.

\subsection{Robot Experiment} \label{sec:exp}
\begin{figure}
    \centering
    \begin{subfigure}[t]{.49\linewidth}
        \centering
        \includegraphics[width=\linewidth]{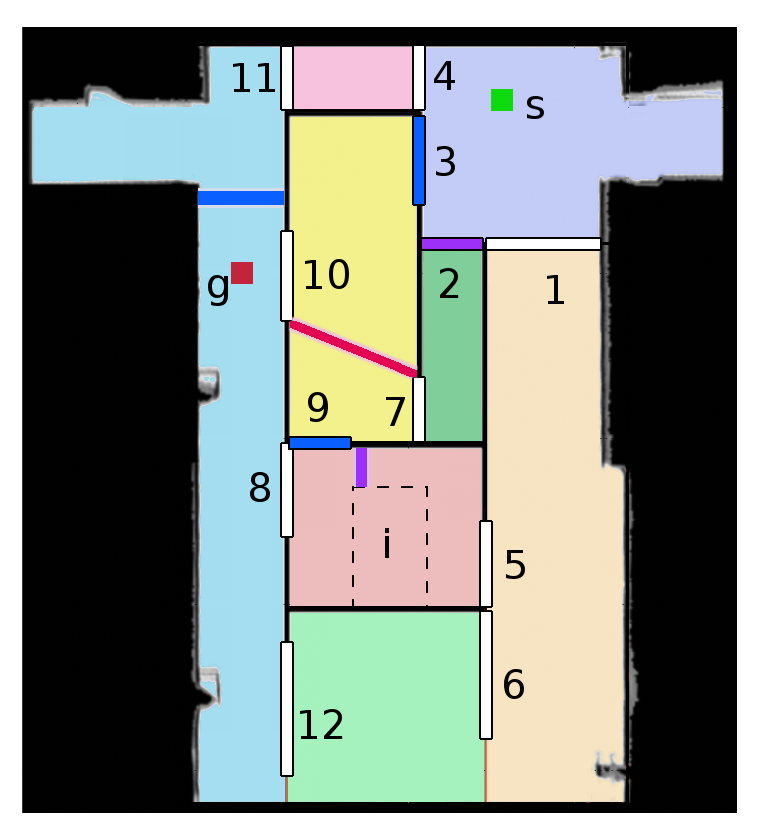}
        \caption{}
        \label{subfig:realmap}
    \end{subfigure}
    \begin{subfigure}[t]{.49\linewidth}
        \centering
        \includegraphics[width=.8\linewidth]{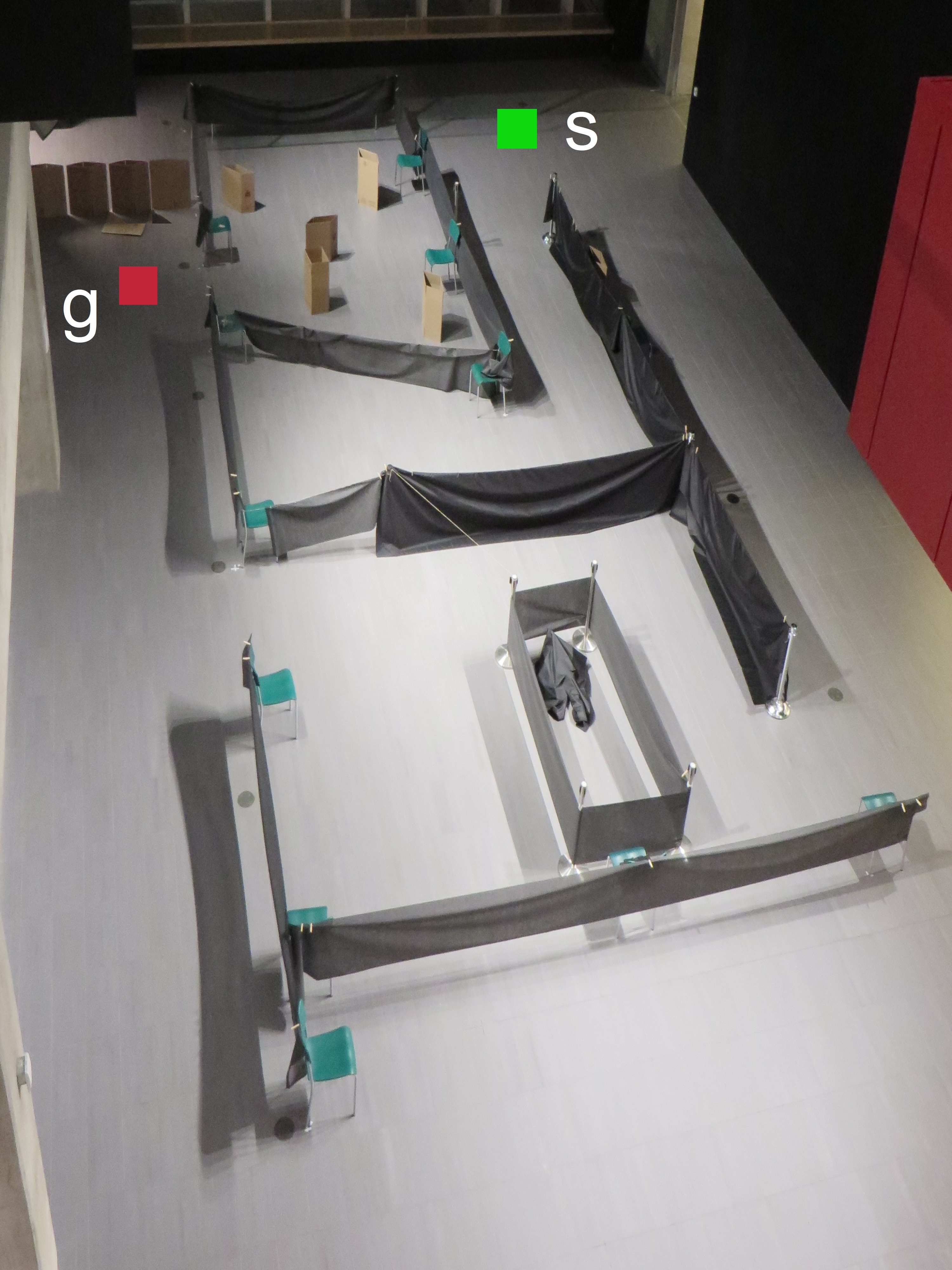}
        \caption{}
        \label{subfig:realex}
    \end{subfigure}
    \caption{(a) Base occupancy grid of real environment (20 m $\times$ 10 m) with submap decomposition, labelled portals, and potential obstacle locations. The start and goal vertices are labelled as $s$ and $g$ respectively. (b) Example of environment configuration with obstacles. In this example, $i$, the dark blue, and pink obstacles exist, along with some debris in the yellow submap.}
    \label{fig:real_ex}
\end{figure}

\begin{table}
    \caption{Probabilities of obstacles appearing in the environment.}
    \label{tab:real_prob}
    \centering
    \begin{tabular}{lccccccc}
         \toprule
         Obstacle &dark blue & \phantom{a}&purple &\phantom{a}&pink &\phantom{a}& $i$ \\
         \midrule
         Probability & 0.6& &0.3&& 0.9&&1.0\\
         \bottomrule
    \end{tabular}
\end{table}

In our physical experiments a Clearpath Jackal was used, similar to the simulation. This was equipped with a Velodyne Puck, a $360^{\circ}$ LiDAR. The environment that the robot operated in is shown in Figure~\ref{fig:real_ex}, which is an event space in the Engineering 7 building on the University of Waterloo campus. The walls in this space are primarily glass, raising several challenges for localization.  The probability of obstacles appearing is shown in Table~\ref{tab:real_prob}. The occupancy grid provided to the robot in Figure~\ref{subfig:realmap} was generated by running the Cartographer SLAM algorithm \cite{Cartographer2016} in the empty area, with the walls of the maze later drawn manually. 

This test environment has $14$ vertices and $35$ edges, with $8$ different realizations. We ran $2$ trials of $10$ tasks, comparing only the Simple and RPP-Hybrid policies. The environment realization for each task execution was randomly selected according to the probabilities in Table~\ref{tab:real_prob}. A selection of these trials are shown in the supplementary video.

The distance travelled by the robot in meters for each task is shown in Table~\ref{tab:real_results_full}. Due to localization errors, the robot incorrectly resolved some edges.  In the table, we mark trials as having ``minor mapping errors'' if between 1 and 4 edges were incorrectly resolved.  We mark trials as having ``major mapping errors'' if more than 4 edges were resolved incorrectly. Trial 1 shows the performance of the algorithm when there are few mapping errors during task execution, which results in a $10.4\%$ improvement in the average travel distance with respect to the simple policy. Trial 2 shows the robustness of the proposed algorithm in the presence of many mapping errors due to non-ideal conditions for sensors and localization. Although $7$ out of $10$ trials were completed with mapping errors, the proposed algorithm still outperforms the simple policy by $5.1\%$. Note that mapping errors in the initial tasks are more detrimental, as it results in an incorrect model of the environment that must be used in later tasks. The effect of this can be seen in task 6, where the robot experiences no mapping errors, but is given a policy that performs worse than the simple policy. While these mapping errors did happen in the simulation, it was rare and they were not as impactful. Because of the continued COVID-19 restrictions, we are able to only present data collected prior to the initial lockdown, thus the limited number of trials. We hope to perform more trials in future work.
\begin{table}
    \caption{Results from robot experiments}
    \label{tab:real_results_full}
    \centering
     \begin{threeparttable}
    \centering
    \begin{tabular}{c  c c c ccc}
         \toprule
         Task &\phantom{a}& \multicolumn{2}{c}{First Trial}&\phantom{a}&\multicolumn{2}{c}{Second Trial}\\
         \cmidrule{3-4}\cmidrule{6-7}
         &&Simple & Ours &&Simple & Ours\\
         \midrule
         1  && 10.2 & 10.7 \hphantom{1} && 74.5   & 60.4 ${\ast}$\\
         2 && 74.5  & 67.4 ${\ast}$&& 70.2 & 39.2 ${\diamond}$\\
         3 && 65.1  & 57.6 \hphantom{1}&& 10.3 & 32.8 ${\diamond}$\\
         4 && 10.2 & 10.8 \hphantom{1}&& 59.7 & 53.5 ${\ast}$\\
         5 && 65.1  & 31.9 \hphantom{1}&& 60.8 & 53.0 \hphantom{1}\\
         6 && 10.2  & 11.4 \hphantom{1}&& 10.1 & 45.4 \hphantom{1}\\
         7 && 10.2  & 10.9 \hphantom{1}&& 23.8   & 13.4 ${\ast}$\\
         8 && 65.1  & 87.8 ${\diamond}$&& 75.1  & 57.3 ${\diamond}$\\
         9 && 10.2  & 10.8 \hphantom{1}&& 22.1  & 42.8 \hphantom{1}\\
         10 && 74.5& 54.8 \hphantom{1}&& 59.9 & 45.0 ${\ast}$\\
         \midrule
         Average& & 39.5  & 35.4 \hphantom{1}&& 46.6 & 44.3  \hphantom{1}\\
         \bottomrule
    \end{tabular}
    
    \smallskip
    \scriptsize
    \begin{tablenotes}
    \item The minor mapping errors in task execution are denoted by $\ast$ and the major mapping errors are denoted by $\diamond$.
    \end{tablenotes}
    \end{threeparttable}
\end{table}

\section{Discussion}
\label{sec:impldisc}

An important consideration is how the environment is decomposed into submaps.  The size of the submaps will affect the performance of the LAMP framework. Increasing the submap size can decrease the size of the resulting navigation graph and reduce the sensitivity of the learning to changes in the environment. However, doing so increases the robot's reliance on the path planning component to navigate, and increases the computation time of the edge observer. Conversely, decreasing the submap size can reduce the computation time of the edge observer to find a path between vertices, but increases the robot's sensitivity to small changes in the environment, causing it to remember more super maps.

As noted in Section~\ref{sec:lrppresults}, there is a possibility for the RPP-Hybrid policy to become trapped in a suboptimal policy. This can happen when the robot encounters an improbable series of realizations that is not representative of the true pmf of the environment. If the robot's perceived probability of an edge is too low, the expected cost of traversing that edge may be high enough that the policy will never explore that edge again, even if its probability estimate of that edge is poor. To remedy this, one could implement a way for the robot to forget old observations such as using a sliding window of observations for policy construction. This would ensure the robot's observations are consistently updated.

In the robot experiments, there were many challenges with localization that we planned to mitigate given more time. For example, there were some windows that could be blocked to improve LiDAR measurements. Another improvement could be to use SLAM to generate the base map with all permanent obstacles present, rather than an idealized floorplan that lacks the additional obstacle texture and depth details.  However, despite these issues, LAMP was still able to learn structure, and exploit it to improve performance.

\section{Conclusions}

This paper introduced the LAMP framework, which can be easily added to an existing navigation stack. This framework allows for the robot to utilize past experience to improve navigation of a familiar environment. We detailed the hybrid map structure that uses an occupancy grid and a navigation graph for the robot to navigate in when using LAMP. We then proposed the edge resolver algorithm so the robot can navigate to the goal using the hybrid map. We reviewed how the Optimistic and RPP-Hybrid policies were constructed, explaining how the high level planner functions. The potential of the LAMP framework to improve navigation was supported by our experimental results, which showed a reduction in average travel cost over time, even with imperfect localization and observations.

One direction of future work would be to draw conclusions about the environment based on more general observations, such as identifying signs, specific objects, or even sounds that could indicate edge traversability. Another area of potential investigation is to examine the impact of using different strategies to choose the best map to merge. Current research is looking into the potential of dynamically forming a policy during task execution based on observations, as opposed to the current approach where a complete policy is calculated before executing the task.

\bibliographystyle{IEEEtran}

\end{document}

%% file: tikzset.tex
\usetikzlibrary{matrix,calc,shapes,arrows.meta,backgrounds,positioning-plus,node-families}

\tikzset{
    box/.style = {
        shape=rectangle, 
        rounded corners, 
        draw, 
        anchor=center, 
        align=center,
        top color=white, 
        bottom color = orange!20, 
        inner sep=1ex},
    process/.style  = {box,
        text width=5em,
        minimum height=1cm
        },
    dummy/.style    = {box, circle, draw},
    decision/.style = {process, diamond, inner sep=0pt},
    root/.style     = {dummy, bottom color=red!30},
    finish/.style   = {dummy, bottom color=green!40},
    connect/.style = {draw, circle, minimum size=1mm, inner sep=0, fill=blue},
    input/.style = {box, 
        trapezium, 
        bottom color=blue!30,
        trapezium left angle=70,
        trapezium right angle=110},
    header node/.style = {
        rectangle,
        font          = \ttfamily,
        text depth    = +0pt,
        fill          = white,
        draw},
    header/.style = {
        inner ysep = +1.5em,
        append after command = {
          \pgfextra{\let\TikZlastnode\tikzlastnode}
          node [header node, anchor=west] (header-\TikZlastnode) at (\TikZlastnode.north west) {#1}
          node [span = (\TikZlastnode)(header-\TikZlastnode)]
            at (fit bounding box) (h-\TikZlastnode) {}},
        },
    hv/.style = {to path = {-|(\tikztotarget)\tikztonodes}},
    vh/.style = {to path = {|-(\tikztotarget)\tikztonodes}},
    control/.style = {blue}
}